\documentclass[preprint,12pt,authoryear]{elsarticle}
\usepackage[margin=1in]{geometry}
\usepackage[normalem]{ulem}
\useunder{\uline}{\ul}{}
\usepackage{hyperref}
\usepackage{amsmath, amssymb, amsthm, enumitem}
\usepackage{amsmath}
\usepackage{algorithm}
\usepackage[table,xcdraw]{xcolor}
\usepackage{graphicx}
\usepackage{subfig}
\usepackage{booktabs}
\usepackage{multirow}
\usepackage{algpseudocode}

\usepackage{amssymb}
\usepackage{amsmath}
\newtheorem{lemma}{Lemma}  

\journal{elsarticle}

\begin{document}

\begin{frontmatter}



\title{CONTINA: Confidence Interval for Traffic Demand Prediction with Coverage Guarantee} 
\begin{highlights}
\item  Propose an adaptive confidence interval modeling method for traffic demand prediction.
\item Prove coverage guarantee of our method for both average and worst-case scenarios.
\item Experiments across 4 datasets demonstrate the effectiveness of our method.
\end{highlights}

\author[a,b]{Chao Yang} 
\author[a]{Xiannan Huang}
\author[a]{Shuhan Qiu}
\author[a]{Yan Cheng \corref{cor1}}
\cortext[cor1]{Corresponding author (e-mail: yan\_cheng@tongji.edu.cn)}
\affiliation[a]{organization={Key Laboratory of Road and Traffic Engineering, Ministry of Education at Tongji University},
            addressline={4800 Cao’an Road}, 
            city={Shanghai},
            postcode={201804}, 
            state={Shanghai},
            country={China}}
\affiliation[b]            {organization={Urban Mobility Institute, Tongji University},
            addressline={1239 Siping Road}, 
            city={Shanghai},
            postcode={200092}, 
            state={Shanghai},
            country={China}}

\begin{abstract}
Accurate short-term traffic demand prediction is critical for the operation of traffic systems. Besides point estimation, the confidence interval of the prediction is also of great importance. Many models for traffic operations, such as shared bike rebalancing and taxi dispatching, take into account the uncertainty of future demand and require confidence intervals as the input. However, existing methods for confidence interval modeling rely on strict assumptions, such as unchanging traffic patterns and correct model specifications, to guarantee enough coverage. Therefore, the confidence intervals provided could be invalid, especially in a changing traffic environment. To fill this gap, we propose an efficient method, CONTINA (Conformal Traffic Intervals with Adaptation) to provide interval predictions that can adapt to external changes. By collecting the errors of interval during deployment, the method can adjust the interval in the next step by widening it if the errors are too large or shortening it otherwise. Furthermore, we theoretically prove that the coverage of the confidence intervals provided by our method converges to the target coverage level. Experiments across four real-world datasets and prediction models demonstrate that the proposed method can provide valid confidence intervals with shorter lengths. Our method can help traffic management personnel develop a more reasonable and robust operation plan in practice. And we release the code, model and dataset in  \href{ https://github.com/xiannanhuang/CONTINA/}{ GitHub}.

\end{abstract}


\end{frontmatter}
\definecolor{myblue}{RGB}{0,0,255}
\section{Introduction}
\label{sec1}
Short-term traffic demand prediction aims to forecast future traffic demand, such as taxi or bike-sharing inflow and outflow, across various urban regions. This task typically involves predicting demand for the next several time steps (e.g., next few hours) based on historical observations and other relevant features. This prediction is crucial as accurate forecasts can help traffic management authorities to allocate resources efficiently, such as rebalancing shared bikes or dispatching taxis \cite{Xu2023MultitaskSP}. Such efforts contribute to alleviating traffic congestion and building a more environmental-friendly society. Substantial work has emerged in this field in recent years, with most focusing on providing more accurate point predictions of future traffic demand. 
\begin{figure}[!h]
\centering
\includegraphics[width=0.99\linewidth]{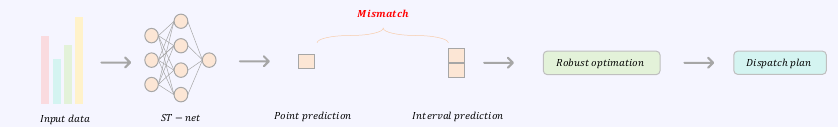}
\caption{The motivation of our paper: point prediction is not enough for may dispatch methods.}
\label{fig:motivation}
\end{figure}

However, point predictions alone are insufficient. Confidence intervals of these predictions are also important because traffic demand inherently involves uncertainty, making it nearly impossible to build a perfectly accurate point prediction model. Therefore, many studies on bike rebalancing or taxi dispatching do not rely solely on point predictions. Instead, they consider confidence intervals as inputs for their models, as shown in Figure \ref{fig:motivation}. These methods often use 
robust optimization \cite{WANG20211,HUANG202390} and assume that future bike usage falls into a certain interval, then return the rebalancing plan that remains efficient based on this assumption \cite{fu2022two,ZHAO2025104933,Yu2024RobustOM,GUO2021161,miao2017data,CHEN2023235}.

From this perspective, the key requirements for predicting confidence intervals are two-fold. 1) \textbf{Validity}. i.e. the model's confidence interval must have a high probability of covering the actual demand in the future. If coverage cannot be guaranteed, future traffic demand may exceed what the rebalancing plan can handle, thus reducing its robustness; and 2) \textbf{Efficiency}. i.e. when ensuring coverage, the confidence interval should be as short as possible. If the confidence interval is too long, the rebalancing plan has to account for highly improbable extreme scenarios, making the plan overly conservative. To this end, it is imperative to provide valid and efficient   confidence intervals for traffic demand prediction with coverage guarantee. 

Recently, some researchers have focused on predicting confidence intervals for traffic demand prediction \cite{Xu2023MultitaskSP,SENGUPTA2024104585,10495765,10458002,WangWZKZ24}. They often use some algorithms like model ensemble \cite{10495765}, Monte Carlo  (MC) Dropout \cite{SENGUPTA2024104585,9310711}, predicting variance and quantile regression \cite{9847117} to construct confidence interval. Although these methods are insightful, they require some rigorous assumptions to ensure validity (i.e. coverage guarantee). First, the functional form of the model, as well as the distributional assumptions about the errors, should be correctly specified. Second, the data distribution during deployment should be consistent with that during training. However, these assumptions may not hold in traffic prediction. Traffic patterns often  change over time \cite{10304591} and could violate the identical distribution assumption. Therefore, confidence intervals derived using these methods may not maintain adequate coverage . Typically, the coverage rate deteriorates with deployment.

To address the issue of overly strict assumptions, some other studies have applied conformal prediction to traffic forecasting \cite{10304591,10472567}, which adopts a dynamic calibration method. But they either focus solely on point predictions or use overly simplistic methods and fail to ensure validity. Therefore, we propose a valid and efficient confidence interval modeling method  by tackling the major challenges in original conformal prediction methods \cite{10.5555} which can provide a valid and efficient confidence interval prediction.

The contributions of our paper are two-fold:
\begin{enumerate}
    \item \textbf{Algorithmically:} We propose CONTINA (Conformal Traffic Intervals with Adaptation), a novel method to provide valid and efficient confidence intervals for traffic demand prediction. By combining adaptive conformal prediction, quantile prediction and dynamic learning rate mechanism, this method is capable of offering shorter confidence interval predictions while maintaining sufficient coverage rates, with evidence from experiments across four datasets.
    \item \textbf{Theoretically:} We prove that by using our approach, with the increase in deployment duration, not only the average coverage of the confidence intervals, but also the coverage of the region with the worst coverage will converge to the target coverage rate.
\end{enumerate}
The rest of this paper is organized as follows. Section 2 reviews the related literature and summarizes the research gaps. Section 3 introduces the method, especially the three improvements made to tackle the challenges in original conformal predictions. Section 4 presents theoretical proof of coverage guarantees. Experiments across four different real-world datasets and results are presented in Section 5. Section 6 concludes this study. Details of proof and results are listed in Appendix. 
\section{Literature review and preliminary}
\subsection{Modeling uncertainty in traffic demand prediction}

Most research on traffic prediction focuses on point forecasting. Early studies employed traditional modeling approaches such as VAR \cite{chandra2009predictions}, Kalman filtering \cite{okutani1984dynamic} VARMA \cite{van1996combining} and prophet \cite{chikkakrishna2022short}. With the advancement of deep learning techniques, spatio-temporal neural networks have increasingly become the mainstream method for traffic prediction \cite{luo2023stg4traffic}, for example LSTnet \cite{li2022lstnet}. Therefore, this section concentrates on how to provide confidence intervals when using neural networks for traffic demand prediction.

In traffic demand prediction, the modeling of confidence intervals typically separates the prediction uncertainty into epistemic uncertainty and aleatoric uncertainty. epistemic uncertainty refers to the mismatch between the patterns captured by the model and the true underlying patterns, which leads to prediction errors. This can be mitigated by increasing the training data or developing more appropriate models. aleatoric uncertainty reflects the inherent uncertainty in the problem itself. A more rigorous definition can be referred to \cite{10.5555/3666122.3666976}.

To obtain epistemic uncertainty, the primary challenge is estimating the probability over models given the training dataset \cite{10.5555/3666122.3666976}. A relatively simplified approach is to use ensemble learning. This involves generating multiple models through different parameter initializations or hyperparameter configurations, then using the variance of predictions across these models to estimate epistemic uncertainty \cite{10495765,NIPS2017_9ef2ed4b,10.5555/3495724.3496270}. This method assumes that training each model is equivalent to sampling from the  distribution of model given then training data. The other approach is to use Bayesian Neural Networks (BNNs), which treats network parameters not as fixed values but as distributions. The posterior distribution of the parameters is inferred from the training data. Since the posterior distribution is often intractable, methods like variational inference \cite{3305890.3305910} or Markov Chain Monte Carlo (MCMC) \cite{10458002,WangWZKZ24,3045118.3045248} are employed to perform sampling.

To obtain aleatoric uncertainty, one method is directly modeling the predicted distribution. For example, the true value is assumed to follow a Gaussian distribution, and the neural network outputs the mean and variance of the distribution \cite{3295222.3295309}. Alternatively, some models assume a negative binomial distribution for the true value and output the corresponding parameters \cite{3583780.3615215}. The other  method to obtain aleatoric uncertainty is outputting confidence interval \cite{pmlr-v80-pearce18a}.

Most studies focusing on confidence intervals in traffic prediction consider both epistemic uncertainty and aleatoric uncertainty to derive the final confidence intervals. Recently, some studies use inherently probabilistic neural networks to output uncertainties,  such as Gaussian Process Regression \cite{Xu2023MultitaskSP,DGP_TITS} or diffusion models \cite{10.1145/3589132.3625614,abs-2401-08119}. 

Recently, deep evidential learning \cite{amini2020deep,meinert2023unreasonable} has emerged as a promising approach for uncertainty quantification in machine learning. Unlike BNNs that place priors on weights, evidential learning places priors directly over the likelihood function's parameters (e.g., mean and variance for Gaussian likelihood). This allows a single deterministic neural network to infer the hyperparameters of a higher-order evidential distribution (such as the Normal-Inverse-Gamma distribution), which naturally encapsulates both aleatoric and epistemic uncertainties without requiring sampling during inference or out-of-distribution data for training \cite{Httel2023DeepEL}.

However these methods have certain limitations. First, these methods assume that the training and test data come from the same distribution, thereby overlooking the dynamic nature of traffic patterns and leading to invalid confidence intervals.  If this assumption is violated, which is common in traffic prediction, these methods would fail. Second, even in a stationary environment, the validity of these methods still relies on strong assumptions \cite{10.5555/3666122.3666976}, including a) large sample size (assuming an infinite amount of training data) and b) model correctness (assuming the model accurately captures the underlying relationships in the training data). 

However, such ideal conditions and assumptions are nearly impossible in real-world scenarios. This raises doubts about their applicability and effectiveness in traffic prediction tasks. Therefore, some researchers in statistics and machine learning fields have introduced conformal prediction to relax the strong assumptions.

\subsection{Conformal prediction}
Conformal prediction includes full conformal prediction and split conformal prediction. Here, we mainly focus on split conformal prediction, whose core idea is to infer the error on the test set using the error on the validation set, thus obtaining the confidence interval for test samples. Specifically, given a validation set $\{(x_i,y_i )\}_{i=1}^n$ and a prediction model $f$, the procedure works as follows: compute prediction error for each data in validation set and gather these errors together, resulting in a set $E=\{|f(x_i )-y_i |:i=1,2,…,n\}$. Then for a test data $x_{n+1}$, the prediction interval $C_{1-\alpha} (x_{n+1})$ can be constructed as \cite{10.5555}:
$$C_{1-\alpha}(x_{n+1}) = \left[ f(x_{n+1}) - Q_{1-\alpha}(E), f(x_{n+1}) + Q_{1-\alpha}(E) \right]$$

where the $Q_{1-\alpha}(E)$ is the $1-\alpha$ quantile of $E$. In detail, $Q_{1-\alpha}(E)$ is the $(1-\alpha)n$-th smallest value in $E$. The theorem vindicating this procedure is based on the assumption of exchangeability, which can be referred to \cite{Lei03072018}.

Although the assumption of exchangeability is weaker than i.i.d. (independent and identically distributed), which is used in many models, it often does not hold in practice. For example, the traffic pattern in future is usually not the same as the pattern in the past. As a result, original conformal prediction methods cannot guarantee the required coverage for confidence intervals in such cases. To tackle this challenge in the context of dynamic forecasting for time series, methods like online conformal prediction have been proposed to improve the original conformal prediction methods. 

The earliest work on online conformal prediction introduced a method to adjust the width of confidence intervals based on their performance during deployment \cite{NEURIPS2021_0d441de7}. For instance, if a confidence interval fails to cover the true value at a given time step, it would be widened for the next step; if it succeeds, it would be narrowed \cite{lin2022conformal}. The rates of widening and narrowing are predefined. Subsequent research extended this idea by removing the need for fixed adjustment rates, proposing adaptive approaches using methods like aggregating experts \cite{pmlr-v162-zaffran22a,JMLR:v25:22-1218} to determine these rates. Some studies framed this as an online convex optimization (OCO) problem \cite{10.1561/2400000013} and used some OCO algorithms to improve interval width adjustments \cite{10.5555/3618408.3618508,Zhang2024TheBO,3692070.3694489}. Beyond adjusting interval widths, other studies focused on dynamically updating the calibration set \cite{10121511}. New data observed during deployment is added to the calibration set, while the oldest data is removed, ensuring the set is updated at each time step. Under certain conditions, this method can also guarantee coverage. Additionally, research has shown that weighting data by similarity to prioritize relevant patterns can also improve coverage \cite{pmlr-v230-jonkers24a,10.1214/23-AOS2276}.

The other direction of improving original conformal prediction  methods focuses on constructing shorter prediction intervals. Traditional methods often yield intervals of uniform length, which can be suboptimal. Some studies have shown that variance differs across data, suggesting that confidence intervals should be longer for high-variance data and shorter for low-variance data \cite{Lei03072018}. Accounting for variance can produce data-specific intervals. Others proposed constructing intervals for each test sample using errors from its nearest neighbors in the calibration set \cite{Lei2014}. Other approaches include partitioning data by features to assign feature-specific interval lengths \cite{10.5555/3692070.3693062}. Replacing point prediction models with conditional distribution prediction models \cite{pnas2107794118,sesia2021conformal} to derive intervals has also been attempted. When data distributions are asymmetric, directly adding or subtracting the same value to the predicted mean is inappropriate. Conformal prediction based on quantile regression \cite{3618408.3620021} has been proposed to address this issue. 

\subsection{Addressing changing patterns for traffic forecasting}
Existing research has noted that changing traffic patterns can cause a data distribution shift between training and test sets. Some studies have proposed methods to address this.
One approach uses adaptive normalization \cite{Kim2022ReversibleIN,Fan2023DishTSAG,xie2023evolving}. Since traffic volume statistics (like mean and variance) differ across time periods, these methods learn specific adaptive parameters for different periods. This ensures data from various times is normalized to a more standard distribution.
Other methods focus on having the model learn latent background variables that capture the underlying time-varying factors affecting traffic patterns \cite{Huang2024}.
Another idea involves online updates to model parameters. For example, \cite{10.5555/3692070.3693235} borrowed a concept from large language models: adding a time-varying "prompt" to the input, which adapts to evolving traffic patterns.
Some studies focus on deployment strategies, like dynamically adjusting a model's parameters after deployment to help it adapt to the changing patterns \cite{chen2024calibration}. Others add an adapter module that learns extra parameters during online deployment to correct the model's output \cite{guo2024online}.

A key point is that these works primarily focus on point forecasting, with less attention paid to confidence intervals. However, they all emphasize that handling time-varying traffic patterns is crucial. This insight directly informs our work: it is essential to account for these dynamic changes when performing probabilistic forecasting.

\subsection{Challenges of conformal prediction in traffic demand forecasting}
Most existing studies on confidence interval modeling for traffic demand forecasting assume that traffic patterns remain unchanged, which is inconsistent with real-world scenarios. Conformal prediction, particularly its extensions, offers an effective way to account for changing traffic patterns when constructing confidence intervals. Recent works have applied conformal prediction methods to model confidence intervals for traffic demand forecasting. However, these approaches are often simplistic, which fails to guarantee coverage under dynamic conditions \cite{10304591}. Additionally, some methods focus solely on pointwise predictions, which is inadequate \cite{10472567}. As a result, there remain some major challenges when adopting original conformal prediction in traffic demand forecasting, which include:
\begin{enumerate}
    \item \textbf{Asymmetric distribution of confidence intervals.} The distribution of traffic demand is not symmetric. Therefore, the method of constructing a confidence interval by adding and subtracting the same value—which is applied when using the absolute residual as the conformity score—might be inappropriate. For example, if the prediction is 7 and the estimated error is 10, then we will get a prediction interval of [-3,17]. The -3 is pretty unreasonable since the traffic demand is at least 0.
    \item \textbf{Dynamic traffic patterns.} In the real world, the traffic pattern is changing over time and data in different time points cannot be considered as exchangeable. This fact will make the confidence interval provided by the original conformal prediction invalid. 
    \item \textbf{Multiple regions.} The traffic demand prediction task is a multivariate task and we need to predict traffic demand in each region. However, the traffic laws might change at different rates in different regions. For example, the traffic law in the region where a new railway station opens might change drastically, but the laws in other regions may not change significantly.
\end{enumerate}

Our work aims to tackle the above mentioned challenges by reforming some conformal prediction methods from machine learning area and developing a conformal prediction framework specifically tailored to traffic demand forecasting problem, with theoretical guarantees. Furthermore, since traffic demand forecasting involves multiple regions, our method seeks to ensure both global and region-specific coverage.
\section{Method}
\subsection{Problem definition}
In this paper, we focus on the problem of predicting the confidence interval for traffic demand at a future time step, given historical traffic demand data from several previous time steps. Suppose there are $n$ regions, and for each region, we need to estimate confidence intervals for both inflow and outflow. Let $y_{t,i,1}$ represent the actual inflow at time $t$ for region $i$, and $y_{t,i,2}$ represent the actual outflow. The estimated confidence bounds are denoted by $[\mathrm{low}_{t,i,1}, \mathrm{up}_{t,i,1}]$ for inflow and $[\mathrm{low}_{t,i,2}, \mathrm{up}_{t,i,2}]$ for outflow. We employ a predictive model $f$ to forecast traffic demand across different regions, aiming to control the confidence level at $1-\alpha$.

Our objective is to provide confidence intervals with sufficient coverage. First, we consider the average coverage. Suppose we deploy our method over $T$ time steps; the average coverage is defined as:
\begin{equation}
\mathrm{cov} = \frac{1}{2nT} \sum_{t=1}^T \sum_{i=1}^n \sum_{j=1}^2 \mathbb{I}\left(y_{t,i,j} \in [\mathrm{low}_{t,i,j}, \mathrm{up}_{t,i,j}]\right)
\label{eq:cov}
\end{equation}
where $\mathbb{I}(\cdot)$ denotes the indicator function, defined as:
$$
\mathbb{I}(a) = \begin{cases} 
1 & \text{if } a \text{ is true} \\
0 & \text{otherwise}
\end{cases}
$$

Additionally, since we provide prediction intervals for each region individually, we also aim to ensure adequate coverage for every single region. The reason for this requirement is that sometimes the overall coverage may meet the target—say, 90\%—but the specific coverage for individual regions can vary significantly. Some regions might have a high coverage, such as 95\%, while others may have a much lower rate, for example, 70\%. As a result, in those regions with lower coverage, the dispatch plan based on the predictions may fail to meet the actual demand. Therefore, we aim to ensure that the coverage requirement is met in every region. This means we want to guarantee that even the region with the lowest coverage maintains a reasonably high coverage. Consequently, we also focus on the following metric (minimum regional coverage, i.e., minRC):
\begin{equation}
\mathrm{minRC} = \min_{i} \left\{ \frac{1}{2T} \sum_{t=1}^T \sum_{j=1}^2 \mathbb{I}\left(y_{t,i,j} \in [\mathrm{low}_{t,i,j}, \mathrm{up}_{t,i,j}]\right) \right\}
\label{eq:minRC}
\end{equation}

As described in the Introduction section, we need to ensure both the \textbf{Validity} and \textbf{Efficiency} of the confidence intervals.
Regarding validity, our targets are to ensure:
\begin{align*}
\mathrm{cov} &= 1-\alpha \\
\mathrm{minRC} &= 1-\alpha
\end{align*}
Regarding efficiency, we aim to provide confidence intervals that are as short as possible. In other words, we need to minimize the following metric ($Length$):
\begin{equation}
\text{Length} = \frac{1}{2nT} \sum_{t=1}^T \sum_{i=1}^n{\sum_{j=1}^2}  |\text{up}_{t,i,j} - \text{low}_{t,i,j}|
\label{eq:length}
\end{equation}

\subsection{Quantile conformal prediction}
As mentioned above, traditional conformal prediction returns symmetric prediction intervals of the form $[\hat{y}-\delta, \hat{y}+\delta]$ when absolute residual is used as the conformity score. However, traffic demand distributions are inherently asymmetric since traffic demand is always non-negative. To address this challenge, we employ quantile conformal prediction \cite{NEURIPS2019_5103c358}, which uses quantile regression to predict different quantiles and handles their asymmetry. We adapt this approach by transforming a point prediction model into a quantile prediction model.

Specifically, we modify the model to predict both the $\alpha/2$ and $1-\alpha/2$ quantiles of traffic demand. This transformation is straightforward for most deep-learning-based traffic prediction models, which typically consist of a spatial-temporal net to excavate features and a prediction head to get the prediction. The only required modification is changing the prediction head's output dimension from 1 to 2 (from predicting just the mean value to predicting both quantiles). During training, we use the quantile loss function:

\begin{align}
\mathcal{L}_{\alpha/2}(y, y_{\alpha/2}) &= (1-\alpha/2)(y_{\alpha/2}-y)\mathbb{I}(y \leq y_{\alpha/2}) + \alpha/2(y-y_{\alpha/2})\mathbb{I}(y > y_{\alpha/2}) \\
\mathcal{L}_{1-\alpha/2}(y, y_{1-\alpha/2}) &= \alpha/2(y_{1-\alpha/2}-y)\mathbb{I}(y \leq y_{1-\alpha/2}) + (1-\alpha/2)(y-y_{1-\alpha/2})\mathbb{I}(y > y_{1-\alpha/2})
\end{align}

where $y$ is the true traffic demand value, $y_{\alpha/2}$ and $y_{1-\alpha/2}$ are the predicted $\alpha/2$ and $1-\alpha/2$ quantiles and $\mathbb{I}(\cdot)$ is the indicator function (1 if condition is true, 0 otherwise)

The total loss is the sum of these individual losses:
$$
\mathcal{L} = \mathcal{L}_{\alpha/2} + \mathcal{L}_{1-\alpha/2}
$$

We compute gradients with respect to the model parameters and update them using standard optimization algorithms.

After training, we adjust the quantile predictions on the validation set following the procedure proposed in \cite{NEURIPS2019_5103c358}. For any given region $i$ and flow direction $j$ (where $j=1$ for inflow and $j=2$ for outflow), consider a data point $(x_{t,i,j}, y_{t,i,j})$ in the validation set with predicted quantiles $y_{t,i,j,\alpha/2}$ (lower) and $y_{t,i,j,1-\alpha/2}$ (upper). The conformal score is calculated as:

\begin{equation}
e_{t,i,j} = \max\left\{y_{t,i,j} - y_{t,i,j,1-\alpha/2}, \ y_{t,i,j,\alpha/2} - y_{t,i,j}\right\}
\label{eq_e}
\end{equation}

This implies:
\begin{itemize}
    \item When $y_{t,i,j} \leq y_{t,i,j,\alpha/2}$, the score becomes $e_{t,i,j} = y_{t,i,j,\alpha/2} - y_{t,i,j}$
    \item When $y_{t,i,j} \geq y_{t,i,j,1-\alpha/2}$, the score becomes $e_{t,i,j} = y_{t,i,j} - y_{t,i,j,1-\alpha/2}$
\end{itemize}

We collect all conformity scores $e_{t,i,j}$ into a set $E_{i,j}$. The $(1-\alpha)$-quantile of $E_{i,j}$ is denoted as $Q_{1-\alpha}(E_{i,j})$. 

For a new test observation $x_{i,j}$, the final prediction interval for $y_{i,j}$ is:

\begin{equation}
C_{1-\alpha}(x_{i,j}) = \left[y_{i,j,\alpha/2} - Q_{1-\alpha}(E_{i,j}), \ y_{i,j,1-\alpha/2} + Q_{1-\alpha}(E_{i,j})\right]
\end{equation}
\begin{figure}
    \centering
    \includegraphics[width=0.95\linewidth]{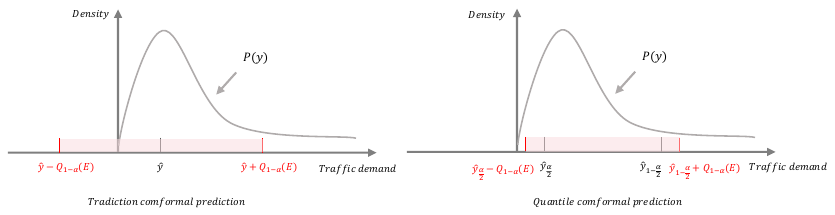}
    \caption{The comparison between traditional conformal prediction ana quantile conformal prediction (The pink area represents confidence interval)}
    \label{fig:qcp}
\end{figure}
where $y_{i,j,\alpha/2}$ and $y_{i,j,1-\alpha/2}$ are the predicted lower and upper quantiles, respectively. The comparison between traditional conformal prediction and quantile conformal prediction is shown in Figure \ref{fig:qcp}, and it can be found that quantile conformal prediction could result in more reasonable confidence interval when the distribution of $y$ is asymmetric. Besides, as proved in \cite{NEURIPS2019_5103c358}, if $y$ is the actual value and data points in the validation and test are exchangeable, then:

$$
\mathbb{P}\left(y_{i,j} \in C_{1-\alpha}(x_{i,j})\right) \geq 1-\alpha
$$

This conclusion means the coverage of quantile conformal prediction can be guaranteed for exchangeable dataset.
\begin{figure}[!b]
\centering
\includegraphics[width=0.99\linewidth]{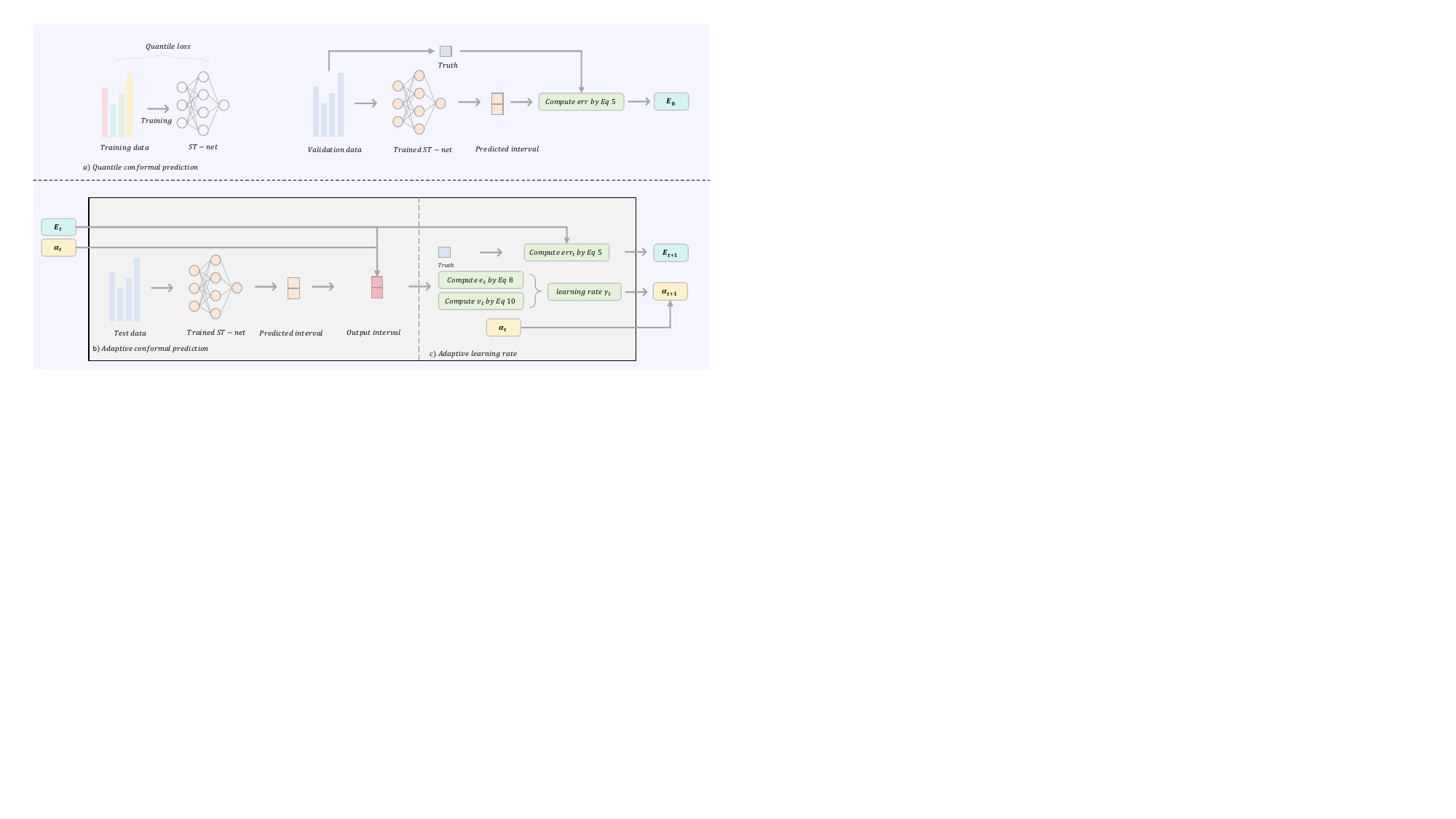}
\caption{The generation procedure of adaptive confidence interval}
\label{fig:workflow}
\end{figure}
\subsection{Dynamic updating of confidence intervals}
Although previous studies have shown that quantile conformal prediction can provide coverage guarantees under the assumption of exchangeable data, our problem involves data that are not exchangeable due to the second challenge, i.e., Dynamic traffic patterns. Therefore, the aforementioned method of using only the validation set to adjust quantile prediction cannot ensure coverage. To address this issue and obtain confidence intervals with coverage guarantee when traffic patterns change, we draw inspiration from adaptive conformal prediction \cite{NEURIPS2021_0d441de7} and propose to update confidence intervals dynamically during deployment based on the actual coverage achieved by the past intervals. For example, if the true demand is not in the confidence interval, then the confidence interval will be elongated in the next time step. Specifically, for a specific region $i$, time step $t$, and flow index $j$ (where $j=1$ for inflow and $j=2$ for outflow), the prediction interval is:

\begin{equation}
C_{1-\alpha}(x_{t,i,j}) = \left[y_{t,i,j,\alpha/2} - Q_{1-\alpha_{t,i}}(E_{t,i,j}), \ y_{t,i,j,1-\alpha/2} + Q_{1-\alpha_{t,i}}(E_{t,i,j})\right]
\label{eq:pred_interval}
\end{equation}

Equation \eqref{eq:pred_interval} replaces the fixed $1-\alpha$ quantile with an adaptive $1-\alpha_{t,i}$ quantile of $E_{t,i,j}$. 

We first explain how $\alpha_{t,i}$ is determined for each time step, then describe the construction of $E_{t,i,j}$.

The parameter $\alpha_{t,i}$ is updated iteratively. At time $t$, we calculate the coverage error as:

\begin{equation}
\mathrm{err}_{t,i} = 1 - \frac{\mathbb{I}(y_{t,i,1} \in C_{1-\alpha}(x_{t,i,1})) + \mathbb{I}(y_{t,i,2} \in C_{1-\alpha}(x_{t,i,2}))}{2}
\label{eq:coverage_error}
\end{equation}

This error represents the proportion of true values falling outside the confidence interval. The parameter $\alpha_{t,i}$ is then updated using:

\begin{equation}
\alpha_{t+1,i} = \alpha_{t,i} + \gamma_{t,i}(\alpha - \mathrm{err}_{t,i})
\label{eq:alpha_update}
\end{equation}

where $\gamma_{t,i}$ serves as a learning rate (discussed in detail later). Intuitively, this update rule compares the observed coverage error with the target $\alpha$. If the actual error exceeds $\alpha$, $\alpha_{t+1,i}$ decreases. For example, if $\alpha_{t,i} = 0.1$ and the observed error ($\mathrm{err}_{t,i}$) is too large. Then $\alpha_{t+1,i}$ will be smaller than $\alpha_{t,i}$, for example, 0.09. This means the 91st percentile quantile will be used instead of the 90th quantile in the next time step, resulting in a wider prediction interval, as the 91st percentile quantile is greater than the 90th percentile quantile.

For cases where the $1-\alpha_{t,i}$ falls outside the interval $[0,1]$, we establish specific rules. Theoretically, when $1-\alpha_{t,i} > 1$, we define $Q_{1-\alpha_{t,i}}(E_{t,i,j}) = +\infty$, and when $1-\alpha_{t,i} < 0$, we define $Q_{1-\alpha_{t,i}}(E_{t,i,j}) = -\infty$. In practical implementation, since infinite values cannot be processed directly, we adopt the following approximations: when $1-\alpha_{t,i} > 1$, we set $Q_{1-\alpha_{t,i}}(E_{t,i,j})$ to be twice the maximum value in $E_{t,i,j}$; when $1-\alpha_{t,i} < 0$, we simply define $C_{1-\alpha}(x_{t,i,j})$ as the empty set.

As for $E_{t,i,j}$, we add the most recent conformal score $e_{t,i,j}$ into $E_{t,i}$, and delete the earliest one in each time step, as suggested by \cite{xu_conformal_2021}.

\subsection{Adaptive learning rate determination}
In Equation \ref{eq:alpha_update}, the learning rate $\gamma_{t,i}$ should be determined. In the earliest adaptive conformal prediction research \cite{NEURIPS2021_0d441de7}, the learning rate was set as a constant. And some later researchers \cite{pmlr-v162-zaffran22a,JMLR:v25:22-1218} pointed out that a constant learning rate could be suboptimal. And the situation in our problem could be even more complicated because the rate of traffic pattern change in different regions might be different. As a result, the learning rate for each region could be distinct. 

To address this challenge, when updating the lengths of confidence intervals for different regions, the rate of updating should vary according to regions. In consequence, we propose a method that decides the updating rate adaptively for each region. We draw inspiration from optimization algorithms for deep learning (such as Adam \cite{Kingma2014AdamAM} and SGDM \cite{1983A}) 
which can set different learning rates for different parameters adaptively, and propose to use second order momentum to adjust the learning rate for each region. This method could keep the learning more stable and accelerate convergence rate \cite{JMLR:v12:duchi11a}. We will elaborate this procedure in the following part.

Suppose the initial learning rate is $\gamma_1$ and initial moment is $v_{1,i}=0$ for region $i$. Then in time step $t$, we have $v_{t-1,i}$ from the past step and obtain $\mathrm{err}_{t,i}$ in this step, then:

\begin{equation}
v_{t,i} = \beta v_{t-1,i} + (1 - \beta)(\mathrm{err}_{t,i} - \alpha)^2 
\label{eq_v}
\end{equation}

Then:

\begin{equation}
\alpha_{t+1,i} = \alpha_{t,i} - \frac{\gamma_1}{\sqrt{v_{t,i}}+\epsilon} (\mathrm{err}_{t,i} - \alpha)
\label{eq_al}
\end{equation}

Therefore, the learning rate at each time step $t$ for each region $i$ is:

$$
\gamma_{t,i} = \frac{\gamma_1}{\sqrt{v_{t,i} } +\epsilon}
$$

where $\epsilon$ is a small number used to prevent dividing by zero.

We summarize our algorithm in Figure \ref{fig:workflow}.

\section{Theoretical results}
\newtheorem{theorem}{Theorem} 
\begin{theorem}[Average guarantee]
For arbitrary prediction models and arbitrary data distributions, we have the following guarantee:
$$
|\mathrm{cov} - (1 - \alpha)| \leq \frac{c}{T}
$$
where $c$ is a constant.
\label{Tho:Avg_cov}
\end{theorem}
This theorem states that, the average coverage achieved by our method will converge to the target coverage as the deployment time increases. Moreover, the convergence rate is inversely proportional to the deployment time. We do not need any unrealistic assumptions, such as i.i.d. data, larger sample size or correct model specification, to ensure coverage.

Additionally, we aim to establish a result indicating that, even in the region with the lowest coverage, our method can still maintain a reasonable level of coverage or converge to the desired target coverage rate. To achieve this result, we need to introduce an additional assumption: the coverage error defined by Equation \ref{eq:coverage_error} at a given time step only depends on the data from the $K$-most recent time steps and is independent of data from earlier time steps. This assumption is not so strict because it is reasonable to assume the traffic demand in pretty early time is independent of the traffic demand in the future, and this will result in the independence of errors.

\begin{theorem}[Coverage guarantee for the worst region]
If for any region $i$, index $j$ and time step $t,t'$ such that $|t'-t|\geq K$, we have:
$$
err_{t,i,j}\perp err_{t',i,j}
$$
then:
$$
\min\mathrm{RC} \geq 1 - \alpha - \frac{c_1}{T} - \sqrt{\frac{c_2 K \log n}{T}}
$$
where $c_1$, $c_2$ are constants, $K$ is the dependence window size, and $n$ is the number of regions.
\label{Tho:min_cov}
\end{theorem}
This theorem demonstrates that even for the region with the lowest coverage, our method can maintain a relatively high coverage level. Furthermore, as the deployment time increases, the coverage rate will converge to the desired target coverage rate. And it is needed to emphasize that the number of regions $n$ just causes error proportion to $\sqrt{log(n)}$, which means that even for a city with a larger number of regions, the worst regional coverage will not deteriorate very much.

The proofs of these two theorems are in \ref{proofs}.
\section{Experiments}
\subsection{Datasets}
In the experiments, we used four real-world datasets. Each dataset spans 16 months, from January to April of the next year.
\begin{enumerate}
    \item \textbf{NYCbike}: This dataset covers shared bike usage records in New York City from January 2022 to April 2023. Each entry represents a bike pickup and drop-off event. Following previous methods in \cite{10.1145/3474717.3483923}, we divided New York into 200 grids and calculated the bike pickup and drop-off quantities for each grid every hour. 
    \item \textbf{NYCtaxi}: This dataset covers taxi usage data from January 2018 to April 2019 in New York City, with hourly usage for each region, including both pick-up and drop-off dimensions. The division of regions is based on the scheme provided by the official website.
    \item \textbf{CHIbike}: The dataset contains shared bike usage records in Chicago from January 2022 to April 2023. Chicago is divided into 200 grids and hourly bike pickup and drop-off quantities for each grid are collected.
    \item \textbf{CHItaxi}: This dataset contains hourly taxi pick-up and drop-off values of census regions in Chicago from January 2016 to April 2017.
\end{enumerate}
The descriptions of datasets are summarized in Table \ref{tab:data_des} and we plot the regions or grids of each dataset in Figure \ref{fig:map}.
\begin{table}[!h]
\centering
\small
\setlength{\tabcolsep}{4pt} 
\begin{tabular}{@{}ccccccc@{}}
\toprule
\multirow{2}{*}{Dataset} & \multirow{2}{*}{No. regions} & \multirow{2}{*}{Mean length/$km$} & \multirow{2}{*}{Mean wide/$km$} & \multirow{2}{*}{Mean area/$km^2$} & \multirow{2}{*}{Mean usage}  \\
                         &                              &                                 &                               &                             &                                                         \\ \midrule
NYCBike                  & 200                          & 1.02                            & 2.89                          & 2.95                        & 17.8                                            \\
NYCTaxi                  & 263                          & 7.59                            & 7.69                          & 32.45                       & 41.8                                              \\
ChiBike                  & 200                          & 0.92                            & 0.82                          & 0.75                        & 2.03                                            \\
ChiTaxi                  & 171                          & 0.97                            & 1.06                          & 0.89                        & 10.3                                                \\ \bottomrule
\end{tabular}
\caption{Descriptions of datasets}
\label{tab:data_des}
\end{table}
\begin{figure}[!h]
    \centering
    \includegraphics[width=0.9\linewidth]{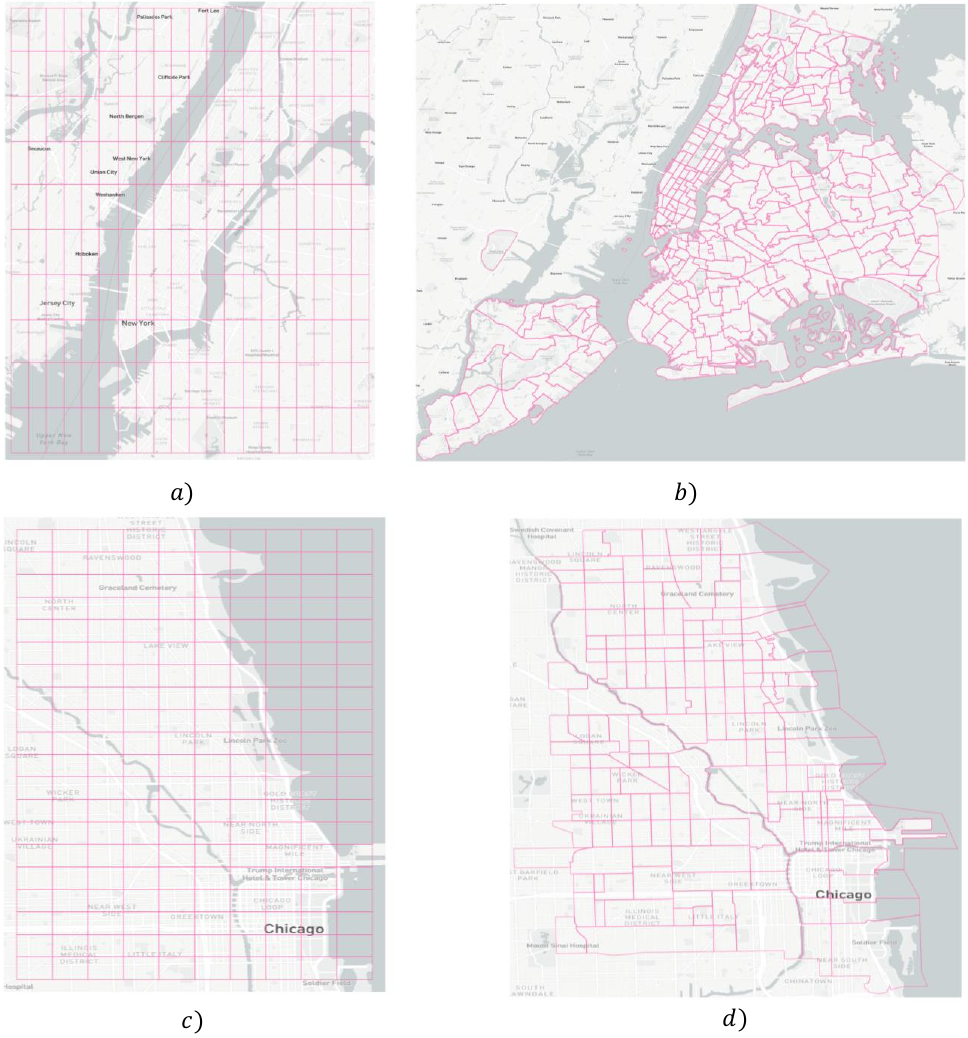}
    \caption{ The grids or regions for datasets. a) the grids of NYCbike dataset, b) the regions of NYCtaxi dataset, c) the grids of CHIbike dataset, d) the regions of CHItaxi datasets. }
    \label{fig:map}
\end{figure}
\subsection{Setup}
The task is predicting bike/taxi usage in the following hour using usage records in the preceding 6 hours. We used data from January to November for training and data in December for validation, then deployed the trained model on data from January to April of the next year. The grids or regions with average bike or taxi usage below 2 were deleted in the experiments.

To validate that our method can be applied to a wide range of models, we selected four classic spatial-temporal prediction models, STGCN \cite{10.5555/3304222.3304273}, DCRNN \cite{li2018diffusion}, MTGNN \cite{10.1145/3394486.3403118} and GWNET \cite{10.5555/3367243.3367303}, as our prediction model. $\alpha$ is set as 0.1, in other words, target coverage rate is 90\%. In the adaptive learning rate algorithm, $\beta$ is set as 0.99, $\gamma_1$ is 0.005, and $\epsilon$ is $e^{-8}$. 
\subsection{Baselines and evaluation metrics}
We choose baseline methods from three perspectives: 
\begin{itemize}
    \item Traditional confidence interval prediction methods: quantile regression (QR), Bootstrap \cite{10495765}, MC dropout \cite{10.5555/3045390.3045502}, directly minimizing Mean interval score (MIS) \cite{10.1145/3447548.3467325}, DeepJMQR \cite{rodrigues2020beyond}.
    \item Methods for confidence interval modeling in traffic prediction task: DESQRUQ \cite{10495765}, ProbGNN \cite{WangWZKZ24}, UATGCN \cite{10458002} and QUANTARFFIC \cite{10304591}. 
    \item Conformal prediction and its online versions: CP (traditional conformal prediction) \cite{10.5555}, ACI (adaptive conformal prediction) \cite{NEURIPS2021_0d441de7}, DtACI \cite{JMLR:v25:22-1218}, QCP (quantile conformal prediction) \cite{NEURIPS2019_5103c358}
\end{itemize}

Regarding evaluation metrics, as we mentioned earlier, confidence intervals must satisfy two key criteria: The intervals should cover the true value with a specified probability. This applies both globally (average coverage across all regions) and locally (coverage for each individual region, ensuring no area is underserved). While meeting the coverage requirement, the intervals should be as short as possible to avoid overly conservative predictions and ensure practical utility. Therefore, we used 3 metrics to evaluate the quality of confidence intervals.

\begin{itemize}
    \item Coverage (Cov): The proportion of true values included within the confidence interval, defined in Equation \ref{eq:cov}.
    \item Minimum Regional Coverage (minRC): The coverage of the confidence interval in the region with the lowest coverage, defined in Equation \ref{eq:minRC}
    \item Length: The average length of the confidence interval, defined in Equation \ref{eq:length}.
\end{itemize}

Besides, in some confidence interval related literature \cite{WangWZKZ24,10495765}, coverage is referred to as ICP (Interval Coverage Probability), and length is referred to as MIL (Mean Interval Length).
\subsection{Results}
\definecolor{myblue}{HTML}{D9E1F4}
\definecolor{myred}{HTML}{FF0000}
\definecolor{mygreen}{HTML}{7030A0}
The results of our experiments are summarized in Table \ref{tab:res_main}. The result with coverage greater than 88\% and minimum regional coverage greater than 85\% is considered as valid and we color the cells of valid results in \textcolor{myblue}{blue}. Among these valid results, the minimum length is expressed in \textcolor{myred}{\textbf{red}} text and the second minimum length is expressed in \textcolor{mygreen}{\ul \textbf{purple}} text with an underline. Besides, the results in Table \ref{tab:res_main} are the average values among all four prediction models, and the full results can be found in \ref{A:full_res}.
\newpage
\begin{table}[!h]
\centering
\setlength{\tabcolsep}{1.pt}
\renewcommand{\arraystretch}{1.}
\fontsize{5.}{7.5}\selectfont 

\begin{tabular}{@{}c|cccccccccccccccc@{}}
\toprule
Dataset                    & Time                       & metric & QR     & MCD    & boostrop & MIS                                                                 & DESQRUQ                                                             & UAT-GCN                                                             & Prob-GNN                       & DeepJMQR & QuanTraffic                    & CP                             & ACI                                                                 & QCP                                                           & DtACI                                                               & CONTINA                                                             \\ \midrule
                           &                            & cov    & 89.6\% & 54.8\% & 30.8\%   & 88.7\%                                                              & \cellcolor[HTML]{D9E1F4}91.7\%                                      & \cellcolor[HTML]{D9E1F4}91.7\%                                      & \cellcolor[HTML]{D9E1F4}93.1\% & 89.5\%   & \cellcolor[HTML]{D9E1F4}91.3\% & 89.2\%                         & \cellcolor[HTML]{D9E1F4}89.8\%                                      & \cellcolor[HTML]{D9E1F4}90.0\%                                & \cellcolor[HTML]{D9E1F4}89.6\%                                      & \cellcolor[HTML]{D9E1F4}89.6\%                                      \\
\bottomrule                                                                
\end{tabular}
\caption{Results of experiments }
\label{tab:res_main}
\end{table}
First, from an overall perspective, our method achieved the best performance in 19 out of 20 cases and the second-best result in the remaining one case. This validates the effectiveness of our proposed approach, which is capable of achieving shorter confidence intervals while maintaining coverage.

Next, we report the experimental results for each dataset in detail. For the NYCbike dataset, several methods (DESQRUQ, UATGCN, ProbGNN) are able to maintain validity during the first month; however, by the last month, only three conformal prediction-based methods (ACI, DtACI, CONTINA) could produce valid confidence intervals. This indicates that the effectiveness of traditional confidence interval construction methods may gradually diminish over time. Additionally, for many confidence interval construction methods (QR, MIS), although they sometimes achieve a coverage rate of near 90\% on average, their performance can be poor in the worst-case region, even below 75\%.

The situation for the NYCtaxi dataset is similar to that of the NYCbike dataset. While some methods (DESQRUQ, UATGCN, ProbGNN) could achieve the target coverage rate on average, their coverage in the worst-performing region are significantly lower than required. Among all valid methods, our approach consistently produced the shortest confidence intervals.

For the two datasets from Chicago, our method also achieved the shortest confidence intervals while ensuring coverage guarantees. Compared to the datasets from New York, the baseline methods show slightly more competitive performance. Notably, many baseline methods are able to maintain validity, particularly in the worst-performing region, where their coverage rates do not drop significantly. This suggests that the heterogeneity of traffic patterns across different regions in the Chicago Bike and Chicago Taxi datasets may not be as pronounced as in the New York datasets. Furthermore, the decline in coverage rates over time for the baseline methods is less significant in the Chicago datasets, indicating that the traffic pattern changes in the Chicago datasets may not be as substantial as those in the New York datasets.

Finally, the average coverage and regional minimum coverage obtained by our approach is always greater 89\% and 88\%, respectively, which cannot be achieved by the any other method. This demonstrates the effectiveness of our approach to maintain coverage. And if we adjust the threshold for regional minimum coverage from 85\% to 86\%, our method can provide the best result in 20 out of 20 cases.
\subsection{Method analysis}
\subsubsection{Sensitive Analysis of initial learning rate}
We conducted an additional sensitivity analysis for the initial learning rate, $\gamma_1$. Specifically, we adjusted $\gamma_1$ to four different values: 0.001, 0.002, 0.005, and 0.01, and repeated our experiments. The results of NYCbike dataset are presented in the following Figure \ref{fig:sen_nycbike}. 
\begin{figure}[!h]
    \centering
    \includegraphics[width=\linewidth]{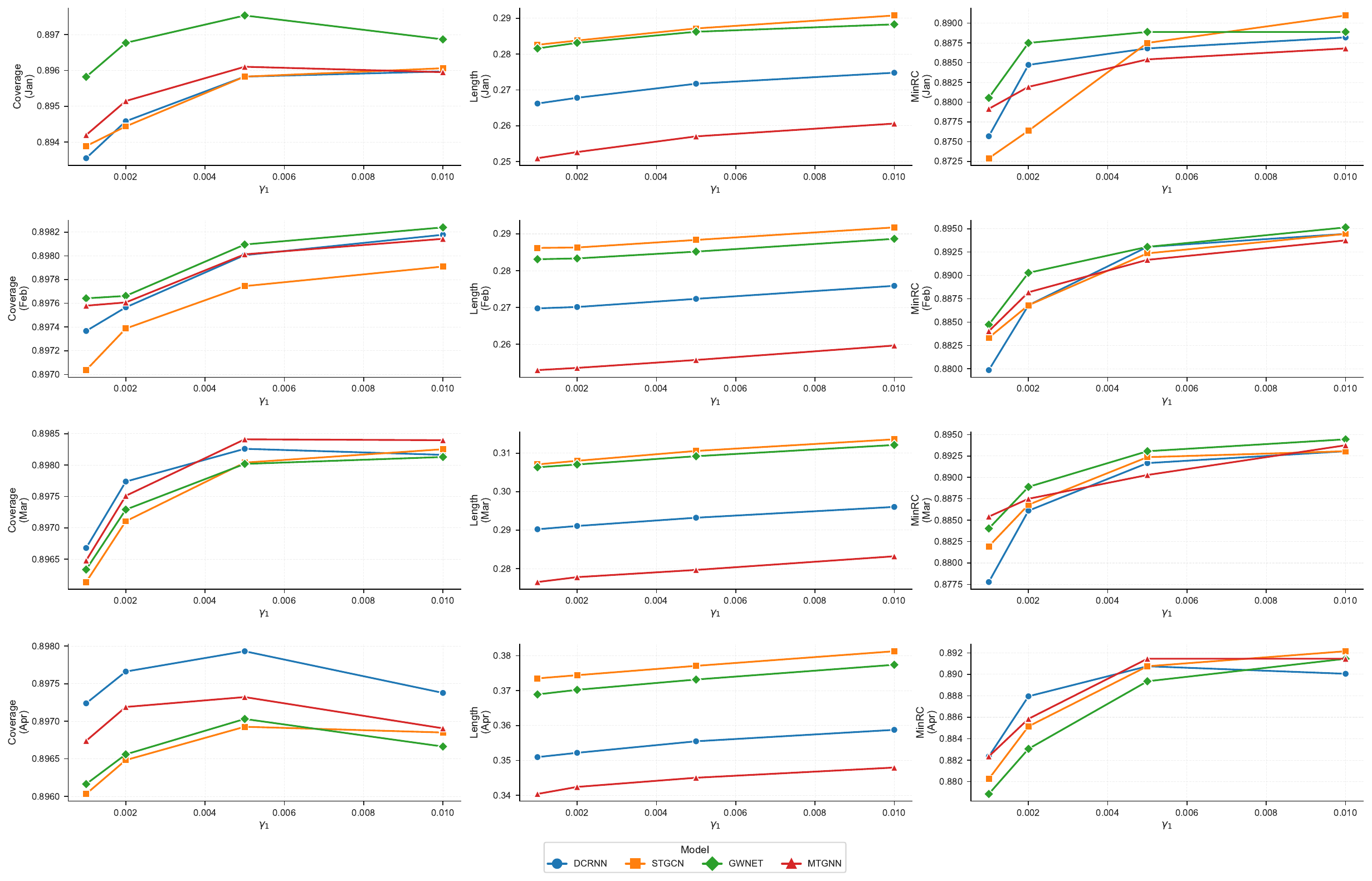}
    \caption{Results of using different initial rates in NYCbike dataset}
    \label{fig:sen_nycbike}
\end{figure}

It can be observed that our method is relatively robust to different initial learning rates. When the initial learning rate is adjusted, the average coverage rate obtained by our method remains above 89\%, and the coverage rate in the worst-case regions stays above 87.5\%. The length of the returned intervals tends to become longer as the learning rate increases, but the magnitude of this change is very limited. The Sensitive Analysis of all datasets is provided in \ref{section:sa_full}.
\subsubsection{The benefit of using adaptive learning rate}
\label{section:main_dr}
To validate whether using an adaptive learning rate for different regions improves the results of the confidence interval, we conducted additional experiments by replacing the adaptive learning rates with a fixed learning rate. We summarized the coverage of the confidence interval for each region and each day, and plotted the results in Figure \ref{fig:day_gwn1} to Figure \ref{fig:day_gwn4}. The solid lines represent the mean coverage across all regions in each day. We also calculate the standard deviation of coverage across regions in each day and plot it in shadow part. (These Figure \ref{fig:day_gwn1} to Figure \ref{fig:day_gwn4} show the situations where GWNET was used as prediction model and the results for other prediction models are in \ref{section:dr}.)
\begin{figure}
    \centering
    \includegraphics[width=\linewidth]{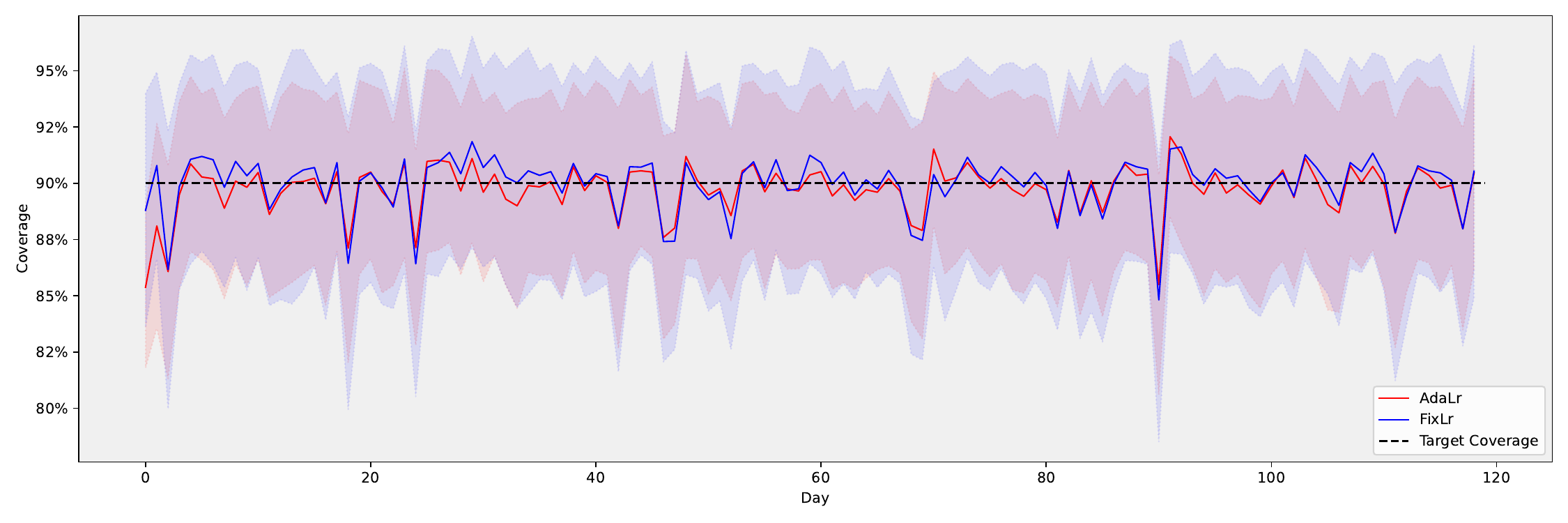}
    \caption{Daily regional coverage for NYCbike dataset}
    \label{fig:day_gwn1}
\end{figure}
\begin{figure}
    \centering
    \includegraphics[width=\linewidth]{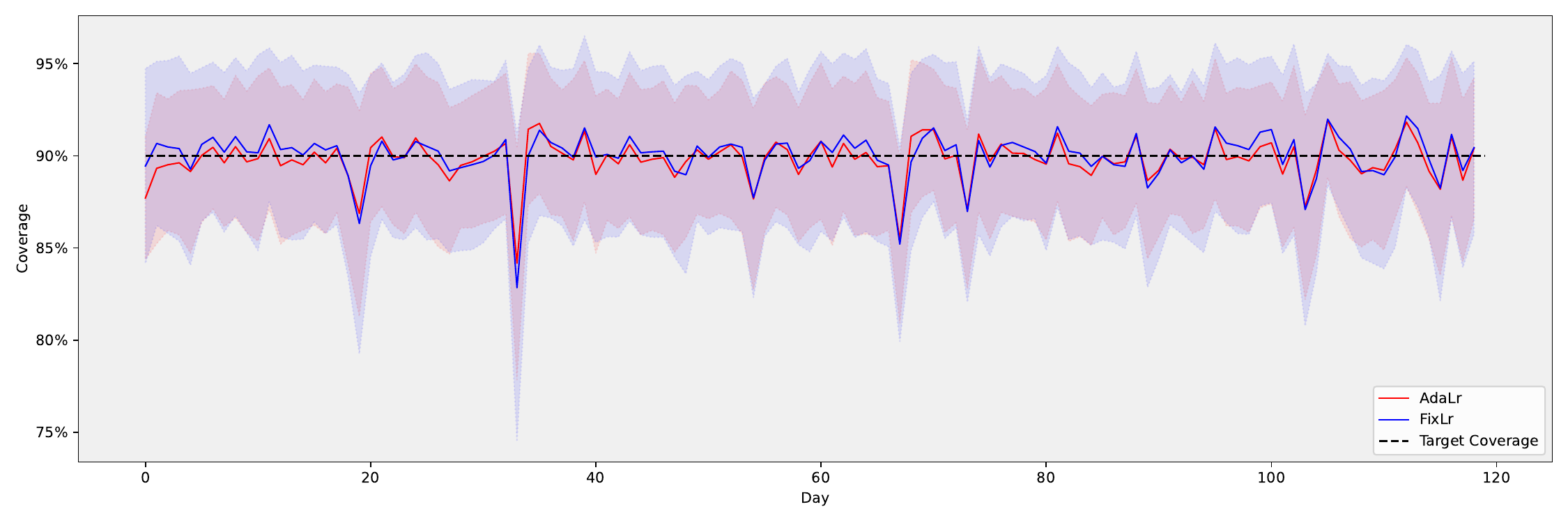}
    \caption{Daily regional coverage for NYCtaxi dataset}
    \label{fig:day_gwn2}
\end{figure}
\begin{figure}
    \centering
    \includegraphics[width=\linewidth]{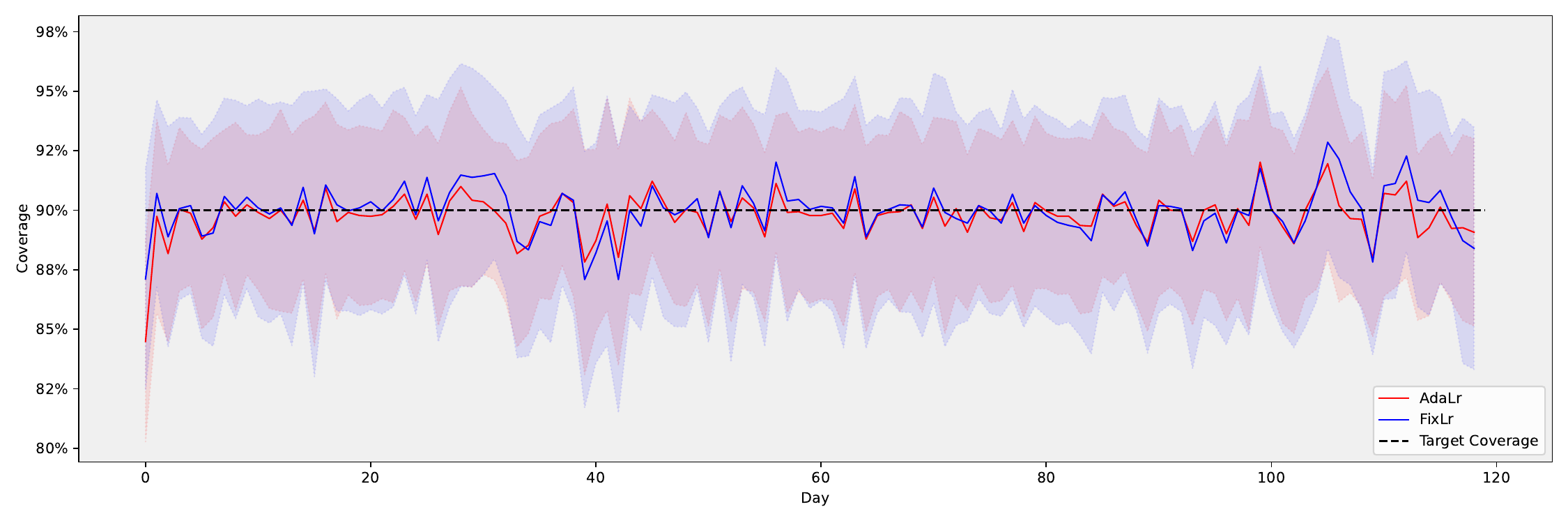}
    \caption{Daily regional coverage for CHIbike dataset}
    \label{fig:day_gwn3}
\end{figure}
\begin{figure}
    \centering
    \includegraphics[width=\linewidth]{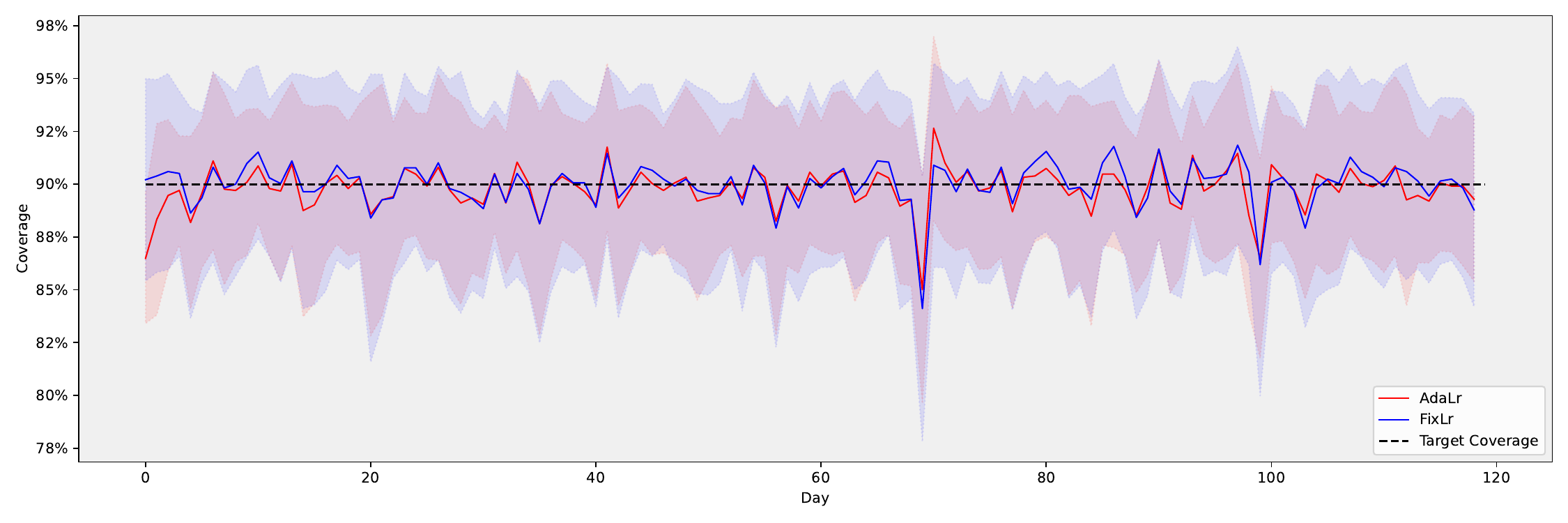}
    \caption{Daily regional coverage for CHItaxi dataset}
    \label{fig:day_gwn4}
\end{figure}

The results suggest that when using an adaptive learning rate, the coverage for all regions is more concentrated around 90\%. However, when using a fixed learning rate, the coverages across regions are more spread out. This indicates that while the overall coverage may still be around 90\%, some regions may have a coverage rate greater than 95\%, while others may have a coverage rate below 85\%. This disparity suggests that fixed learning rate cannot handle the varying transportation patterns in different regions as well as adaptive learning rate. For some regions, the learning rate could be too fast, while for others, it could be too slow. In conclusion, using an adaptive learning rate results in more stable performance, with coverage more consistently aligned with the target rate.
\section{Conclusion}
In this paper, a valid and efficient method for constructing confidence intervals for traffic demand prediction is proposed. To overcome changes in traffic patterns, an adaptive quantile conformal prediction method is introduced. Besides, an adaptive learning rate scheme is used to manage the heterogeneity of traffic changes across different regions. Unlike traditional approaches, the proposed method does not require some strict conditions to ensure coverage. 

Theoretical guarantees for the proposed method are also provided. First, our method ensures that the desired coverage rate can be achieved at the citywide level. Second, even for the regions with the lowest coverage, our method delivers satisfactory performance. Furthermore, both overall and regional coverage rates converge to the desired levels as the deployment period increases. Experiments were conducted using real-world data from bike-sharing and taxi systems to validate the effectiveness of the proposed method. 
\definecolor{myblue}{RGB}{0,0,255}

In the future, we plan to propose some methods to construct confidence intervals for multiple-step prediction problems and try to use the feature of deep learning models to construct more precise confidence intervals. As for theoretical aspects, how to provide confidence interval with conditional coverage or less conditional coverage error could be a promising research question. Furthermore, the current proof for minRC relies on the assumption that errors with large time intervals are independent of each other. This assumption may not hold perfectly in practice. In the future, a weaker assumption, such as strong mixing \cite{2007Introduction}, could be employed to derive the minRC guarantee.
Additionally, while our experiments show that the performance is not highly sensitive to the choice of learning rate, tuning this parameter remains a potential challenge in practical deployment. One promising direction is to adopt an ensemble approach that combines multiple learning rates like DtACI \cite{JMLR:v25:22-1218}

Furthermore, enhancing CONTINA's capability in sparse data scenarios through smoothing techniques and meta-learning represents a critical direction for future improvement. For example, when updating parameters for regions with sparse feedback, incorporating information from neighboring areas or analogous time periods can provide more stable pseudo-feedback. Additionally, implementing a meta-learning framework would enable the system to rapidly adapt update strategies for data-sparse regions by leveraging patterns learned across all available regions.

Besides, several recent studies have explored integrating spatial relationships into conformal prediction. For instance, some approaches utilize prediction errors from other regions to inform the confidence intervals of a given region \cite{jiang2024spatio,jiang2024spatial}, while others incorporate the errors of neighboring regions into the nonconformity score calculation of the target region \cite{wang2025non,song2024similarity}. Additionally, there are works on multivariate conformal prediction that leverage covariance across multiple sequences to derive more reasonable confidence intervals \cite{xu2024conformal,sun2022copula}. Future work can draw inspiration from these methodologies to incorporate spatial relationships into our CONTINA framework.
\definecolor{myblue}{HTML}{D9E1F4}

In conclusion, the proposed approach has practical applications in traffic operation, such as shared-bike rebalancing or taxi dispatching. Additionally, the method allows for real-time monitoring of prediction interval widths. When the intervals become excessively wide, practitioners are alerted that the current model no longer reflects traffic patterns adequately and needs some updates.

\bibliographystyle{elsarticle-harv} 
\bibliography{ref}

@article{Xu2023MultitaskSP,
  title={Multi-task supply-demand prediction and reliability analysis for docked bike-sharing systems via transformer-encoder-based neural processes},
  author={Meng Xu and Yining Di and Han Yang and Xiqun (Michael) Chen and Zheng Zhu},
  journal={Transportation Research Part C: Emerging Technologies},
  year={2023},

}

@article{fu2022two,
  title={A two-stage robust approach to integrated station location and rebalancing vehicle service design in bike-sharing systems},
  author={Fu, Chenyi and Zhu, Ning and Ma, Shoufeng and Liu, Ronghui},
  journal={European Journal of Operational Research},
  volume={298},
  number={3},
  pages={915--938},
  year={2022},
  publisher={Elsevier}
}

@article{ZHAO2025104933,
title = {Research on rebalancing of large-scale bike-sharing system driven by zonal heterogeneity and demand uncertainty},
journal = {Transportation Research Part C: Emerging Technologies},
volume = {170},
pages = {104933},
year = {2025},
issn = {0968-090X},
doi = {https://doi.org/10.1016/j.trc.2024.104933},
url = {https://www.sciencedirect.com/science/article/pii/S0968090X24004546},
author = {Rui Zhao and Zihao Tian and Lixin Tian and Wenshan Liu and David Z.W. Wang},
}

@article{Yu2024RobustOM,
  title={Robust Optimization Model Based on a Static Rebalancing Design for Bike- Sharing Systems Affected by Demand Uncertainty},
  author={Lijun Yu and Yaogeng Xu and Haochen Shi and Zhimin Zhang},
  journal={IEEE Access},
  year={2024},
  volume={12},
  pages={73936-73949},

}

@article{miao2017data,
  title={Data-driven robust taxi dispatch under demand uncertainties},
  author={Miao, Fei and Han, Shuo and Lin, Shan and Wang, Qian and Stankovic, John A and Hendawi, Abdeltawab and Zhang, Desheng and He, Tian and Pappas, George J},
  journal={IEEE Transactions on Control Systems Technology},
  volume={27},
  number={1},
  pages={175--191},
  year={2017},
  publisher={IEEE}
}

@article{GUO2021161,
title = {Robust matching-integrated vehicle rebalancing in ride-hailing system with uncertain demand},
journal = {Transportation Research Part B: Methodological},
volume = {150},
pages = {161-189},
year = {2021},
issn = {0191-2615},
doi = {https://doi.org/10.1016/j.trb.2021.05.015},
url = {https://www.sciencedirect.com/science/article/pii/S0191261521001004},
author = {Xiaotong Guo and Nicholas S. Caros and Jinhua Zhao},
}

@article{SENGUPTA2024104585,
title = {A Bayesian approach to quantifying uncertainties and improving generalizability in traffic prediction models},
journal = {Transportation Research Part C: Emerging Technologies},
volume = {162},
pages = {104585},
year = {2024},
issn = {0968-090X},
doi = {https://doi.org/10.1016/j.trc.2024.104585},
url = {https://www.sciencedirect.com/science/article/pii/S0968090X24001062},
author = {Agnimitra Sengupta and Sudeepta Mondal and Adway Das and S. Ilgin Guler},
}

@ARTICLE{10495765,
  author={Mallick, Tanwi and Macfarlane, Jane and Balaprakash, Prasanna},
  journal={IEEE Transactions on Intelligent Transportation Systems}, 
  title={Uncertainty Quantification for Traffic Forecasting Using Deep-Ensemble-Based Spatiotemporal Graph Neural Networks}, 
  year={2024},
  volume={25},
  number={8},
  pages={9141-9152},
  doi={10.1109/TITS.2024.3381099}}

@ARTICLE{10458002,
  author={Qian, Weizhu and Nielsen, Thomas Dyhre and Zhao, Yan and Larsen, Kim Guldstrand and Yu, James Jianqiao},
  journal={IEEE Transactions on Intelligent Transportation Systems}, 
  title={Uncertainty-Aware Temporal Graph Convolutional Network for Traffic Speed Forecasting}, 
  year={2024},
  volume={25},
  number={8},
  pages={8578-8590},
  doi={10.1109/TITS.2024.3365721}}

@inproceedings{3045118.3045248,
author = {Salimans, Tim and Kingma, Diederik P. and Welling, Max},
title = {Markov Chain Monte Carlo and variational inference: bridging the gap},
year = {2015},
publisher = {JMLR.org},
booktitle = {Proceedings of the 32nd International Conference on International Conference on Machine Learning - Volume 37},
pages = {1218–1226},
numpages = {9},
location = {Lille, France},
series = {ICML'15}
}

@article{WangWZKZ24,
  author={Qingyi Wang and Shenhao Wang and Dingyi Zhuang and Haris N. Koutsopoulos and Jinhua Zhao},
  title={Uncertainty Quantification of Spatiotemporal Travel Demand With Probabilistic Graph Neural Networks},
  year={2024},
  month={August},
  cdate={1722470400000},
  journal={IEEE Transactions on Intelligent Transportation Systems},
  volume={25},
  number={8},
  pages={8770-8781},
  url={https://doi.org/10.1109/TITS.2024.3367779}
}

@inproceedings{10.5555/3045390.3045502,
author = {Gal, Yarin and Ghahramani, Zoubin},
title = {Dropout as a Bayesian approximation: representing model uncertainty in deep learning},
year = {2016},
publisher = {JMLR.org},
booktitle = {Proceedings of the 33rd International Conference on International Conference on Machine Learning - Volume 48},
pages = {1050–1059},
numpages = {10},
location = {New York, NY, USA},
series = {ICML'16}
}

@inproceedings{10.1145/3447548.3467325,
author = {Wu, Dongxia and Gao, Liyao and Chinazzi, Matteo and Xiong, Xinyue and Vespignani, Alessandro and Ma, Yi-An and Yu, Rose},
title = {Quantifying Uncertainty in Deep Spatiotemporal Forecasting},
year = {2021},
isbn = {9781450383325},
publisher = {Association for Computing Machinery},
address = {New York, NY, USA},
url = {https://doi.org/10.1145/3447548.3467325},
doi = {10.1145/3447548.3467325},
booktitle = {Proceedings of the 27th ACM SIGKDD Conference on Knowledge Discovery \& Data Mining},
pages = {1841–1851},
numpages = {11},
location = {Virtual Event, Singapore},
series = {KDD '21}
}

@ARTICLE{9310711,
  author={Li, Can and Bai, Lei and Liu, Wei and Yao, Lina and Waller, S Travis},
  journal={IEEE Transactions on Intelligent Transportation Systems}, 
  title={Graph Neural Network for Robust Public Transit Demand Prediction}, 
  year={2022},
  volume={23},
  number={5},
  pages={4086-4098},
  doi={10.1109/TITS.2020.3041234}}

@inproceedings{3295222.3295309,
author = {Kendall, Alex and Gal, Yarin},
title = {What uncertainties do we need in Bayesian deep learning for computer vision?},
year = {2017},
isbn = {9781510860964},
publisher = {Curran Associates Inc.},
address = {Red Hook, NY, USA},
booktitle = {Proceedings of the 31st International Conference on Neural Information Processing Systems},
pages = {5580–5590},
numpages = {11},
location = {Long Beach, California, USA},
series = {NIPS'17}
}

@inproceedings{3583780.3615215,
author = {Jiang, Xinke and Zhuang, Dingyi and Zhang, Xianghui and Chen, Hao and Luo, Jiayuan and Gao, Xiaowei},
title = {Uncertainty Quantification via Spatial-Temporal Tweedie Model for Zero-inflated and Long-tail Travel Demand Prediction},
year = {2023},
isbn = {9798400701245},
publisher = {Association for Computing Machinery},
address = {New York, NY, USA},
url = {https://doi.org/10.1145/3583780.3615215},
doi = {10.1145/3583780.3615215},
booktitle = {Proceedings of the 32nd ACM International Conference on Information and Knowledge Management},
pages = {3983–3987},
numpages = {5},
location = {Birmingham, United Kingdom},
series = {CIKM '23}
}

@InProceedings{pmlr-v80-pearce18a,
  title = 	 {High-Quality Prediction Intervals for Deep Learning: A Distribution-Free, Ensembled Approach},
  author =       {Pearce, Tim and Brintrup, Alexandra and Zaki, Mohamed and Neely, Andy},
  booktitle = 	 {Proceedings of the 35th International Conference on Machine Learning},
  pages = 	 {4075--4084},
  year = 	 {2018},
  editor = 	 {Dy, Jennifer and Krause, Andreas},
  volume = 	 {80},
  series = 	 {Proceedings of Machine Learning Research},
  month = 	 {10--15 Jul},
  publisher =    {PMLR},
}

@ARTICLE{DGP_TITS,
  author={Jiang, Yunliang and Fan, Jinbin and Liu, Yong and Zhang, Xiongtao},
  journal={IEEE Transactions on Intelligent Transportation Systems}, 
  title={Deep Graph Gaussian Processes for Short-Term Traffic Flow Forecasting From Spatiotemporal Data}, 
  year={2022},
  volume={23},
  number={11},
  pages={20177-20186},
  doi={10.1109/TITS.2022.3178136}}

@ARTICLE{9847117,
  author={Li, Yiqun and Chai, Songjian and Wang, Guibin and Zhang, Xian and Qiu, Jing},
  journal={IEEE Transactions on Intelligent Transportation Systems}, 
  title={Quantifying the Uncertainty in Long-Term Traffic Prediction Based on PI-ConvLSTM Network}, 
  year={2022},
  volume={23},
  number={11},
  pages={20429-20441},
  doi={10.1109/TITS.2022.3193184}}

@ARTICLE{10304591,
  author={Wu, Ying and Ye, Yongchao and Zeb, Adnan and Yu, James Jianqiao and Wang, Zheng},
  journal={IEEE Transactions on Intelligent Transportation Systems}, 
  title={Adaptive Modeling of Uncertainties for Traffic Forecasting}, 
  year={2024},
  volume={25},
  number={5},
  pages={4427-4442},
  doi={10.1109/TITS.2023.3327100}}

@ARTICLE{10472567,
  author={Laña, Ibai and Olabarrieta, Ignacio and Ser, Javier Del},
  journal={IEEE Transactions on Intelligent Transportation Systems}, 
  title={Measuring the Confidence of Single-Point Traffic Forecasting Models: Techniques, Experimental Comparison, and Guidelines Toward Their Actionability}, 
  year={2024},
  volume={25},
  number={9},
  pages={11180-11199},
  doi={10.1109/TITS.2024.3375936}}

@article{10.5555,
author = {Shafer, Glenn and Vovk, Vladimir},
title = {A Tutorial on Conformal Prediction},
year = {2008},
issue_date = {6/1/2008},
publisher = {JMLR.org},
volume = {9},
issn = {1532-4435},
journal = {Journal of Machine Learning Research},
month = jun,
pages = {371–421},
numpages = {51}
}

@inproceedings{10.5555/3666122.3666976,
author = {Schweighofen, Kajetan and Aichberger, Lukas and Ielanskyi, Mykyta and Klambauer, G\"{u}nter and Hochreiter, Sepp},
title = {Quantification of uncertainty with adversarial models},
year = {2023},
publisher = {Curran Associates Inc.},
address = {Red Hook, NY, USA},
booktitle = {Proceedings of the 37th International Conference on Neural Information Processing Systems},
articleno = {854},
numpages = {39},
location = {New Orleans, LA, USA},
series = {NIPS '23}
}

@inproceedings{NIPS2017_9ef2ed4b,
 author = {Lakshminarayanan, Balaji and Pritzel, Alexander and Blundell, Charles},
 booktitle = {Advances in Neural Information Processing Systems},
 editor = {I. Guyon and U. Von Luxburg and S. Bengio and H. Wallach and R. Fergus and S. Vishwanathan and R. Garnett},
 pages = {},
 publisher = {Curran Associates, Inc.},
 title = {Simple and Scalable Predictive Uncertainty Estimation using Deep Ensembles},
 url = {https://proceedings.neurips.cc/paper_files/paper/2017/file/9ef2ed4b7fd2c810847ffa5fa85bce38-Paper.pdf},
 volume = {30},
 year = {2017}
}

@inproceedings{10.5555/3495724.3496270,
author = {Wenzel, Florian and Snoek, Jasper and Tran, Dustin and Jenatton, Rodolphe},
title = {Hyperparameter ensembles for robustness and uncertainty quantification},
year = {2020},
isbn = {9781713829546},
publisher = {Curran Associates Inc.},
address = {Red Hook, NY, USA},
booktitle = {Proceedings of the 34th International Conference on Neural Information Processing Systems},
articleno = {546},
numpages = {14},
location = {Vancouver, BC, Canada},
series = {NIPS '20}
}

@inproceedings{3305890.3305910,
author = {Louizos, Christos and Welling, Max},
title = {Multiplicative normalizing flows for variational Bayesian neural networks},
year = {2017},
publisher = {JMLR.org},
pages = {2218–2227},
numpages = {10},
location = {Sydney, NSW, Australia},
series = {ICML'17}
}

@inproceedings{10.1145/3589132.3625614,
author = {Wen, Haomin and Lin, Youfang and Xia, Yutong and Wan, Huaiyu and Wen, Qingsong and Zimmermann, Roger and Liang, Yuxuan},
title = {DiffSTG: Probabilistic Spatio-Temporal Graph Forecasting with Denoising Diffusion Models},
year = {2023},
isbn = {9798400701689},
publisher = {Association for Computing Machinery},
address = {New York, NY, USA},
url = {https://doi.org/10.1145/3589132.3625614},
doi = {10.1145/3589132.3625614},
booktitle = {Proceedings of the 31st ACM International Conference on Advances in Geographic Information Systems},
articleno = {60},
numpages = {12},
location = {Hamburg, Germany},
series = {SIGSPATIAL '23}
}

@article{abs-2401-08119,
  publtype={informal},
  author={Lequan Lin and Dai Shi and Andi Han and Junbin Gao},
  title={SpecSTG: A Fast Spectral Diffusion Framework for Probabilistic Spatio-Temporal Traffic Forecasting},
  year={2024},
  cdate={1704067200000},
  journal={CoRR},
  volume={abs/2401.08119},
  url={https://doi.org/10.48550/arXiv.2401.08119}
}

@article{Lei03072018,
author = {Jing Lei and Max G’Sell and Alessandro Rinaldo and Ryan J. Tibshirani and Larry Wasserman and},
title = {Distribution-Free Predictive Inference for Regression},
journal = {Journal of the American Statistical Association},
volume = {113},
number = {523},
pages = {1094--1111},
year = {2018},
publisher = {ASA Website},
doi = {10.1080/01621459.2017.1307116},
URL = { 
     https://doi.org/10.1080/01621459.2017.1307116
},
}

@inproceedings{NEURIPS2021_0d441de7,
 author = {Gibbs, Isaac and Candes, Emmanuel},
 booktitle = {Advances in Neural Information Processing Systems},
 editor = {M. Ranzato and A. Beygelzimer and Y. Dauphin and P.S. Liang and J. Wortman Vaughan},
 pages = {1660--1672},
 publisher = {Curran Associates, Inc.},
 title = {Adaptive Conformal Inference Under Distribution Shift},
 url = {https://proceedings.neurips.cc/paper_files/paper/2021/file/0d441de75945e5acbc865406fc9a2559-Paper.pdf},
 volume = {34},
 year = {2021}
}

@inproceedings{
lin2022conformal,
title={Conformal Prediction with Temporal Quantile Adjustments},
author={Zhen Lin and Shubhendu Trivedi and Jimeng Sun},
booktitle={Advances in Neural Information Processing Systems},
editor={Alice H. Oh and Alekh Agarwal and Danielle Belgrave and Kyunghyun Cho},
year={2022},
url={https://openreview.net/forum?id=PM5gVmG2Jj}
}

@InProceedings{pmlr-v162-zaffran22a,
  title = 	 {Adaptive Conformal Predictions for Time Series},
  author =       {Zaffran, Margaux and Feron, Olivier and Goude, Yannig and Josse, Julie and Dieuleveut, Aymeric},
  booktitle = 	 {Proceedings of the 39th International Conference on Machine Learning},
  pages = 	 {25834--25866},
  year = 	 {2022},
  editor = 	 {Chaudhuri, Kamalika and Jegelka, Stefanie and Song, Le and Szepesvari, Csaba and Niu, Gang and Sabato, Sivan},
  volume = 	 {162},
  series = 	 {Proceedings of Machine Learning Research},
  month = 	 {17--23 Jul},
  publisher =    {PMLR},
  url = 	 {https://proceedings.mlr.press/v162/zaffran22a.html},
}

@article{1983A,
  title={A method for solving the convex programming problem with convergence rate $O(1/k^2$)},
  author={ Nesterov, Y. E. },
  journal={Dokl.akad.nauk Sssr},
  volume={269},
  year={1983},
}

@article{JMLR:v12:duchi11a,
  author  = {John Duchi and Elad Hazan and Yoram Singer},
  title   = {Adaptive Subgradient Methods for Online Learning and Stochastic Optimization},
  journal = {Journal of Machine Learning Research},
  year    = {2011},
  volume  = {12},
  number  = {61},
  pages   = {2121--2159},
  url     = {http://jmlr.org/papers/v12/duchi11a.html}
}

@inproceedings{xu_conformal_2021,
	title = {Conformal prediction interval for dynamic time-series},
	url = {https://proceedings.mlr.press/v139/xu21h.html},
	 year    = {2021},
	eventtitle = {International Conference on Machine Learning},
	pages = {11559--11569},
	booktitle = {Proceedings of the 38th International Conference on Machine Learning},
	publisher = {{PMLR}},
	author = {Xu, Chen and Xie, Yao},
	date = {2021-07-01},
	langid = {english},
	note = {{ISSN}: 2640-3498},
}

@inproceedings{10.1145/3474717.3483923,
author = {Wang, Jingyuan and Jiang, Jiawei and Jiang, Wenjun and Li, Chao and Zhao, Wayne Xin},
title = {LibCity: An Open Library for Traffic Prediction},
year = {2021},
isbn = {9781450386647},
publisher = {Association for Computing Machinery},
address = {New York, NY, USA},
url = {https://doi.org/10.1145/3474717.3483923},
doi = {10.1145/3474717.3483923},
pages = {145–148},
numpages = {4},
location = {Beijing, China},
series = {SIGSPATIAL '21},
booktitle = {Proceedings of the 29th International Conference on Advances in Geographic Information Systems},
pages = {145–148},
numpages = {4},
location = {Beijing, China},
series = {SIGSPATIAL '21}
}

@article{Kingma2014AdamAM,
  title={Adam: A Method for Stochastic Optimization},
  author={Diederik P. Kingma and Jimmy Ba},
  journal={CoRR},
  year={2014},
  volume={abs/1412.6980},
}

@article{JMLR:v25:22-1218,
  author  = {Isaac Gibbs and Emmanuel J. Cand{{\`e}}s},
  title   = {Conformal Inference for Online Prediction with Arbitrary Distribution Shifts},
  journal = {Journal of Machine Learning Research},
  year    = {2024},
  volume  = {25},
  number  = {162},
  pages   = {1--36},
  url     = {http://jmlr.org/papers/v25/22-1218.html}
}

@inproceedings{NEURIPS2019_5103c358,
 author = {Romano, Yaniv and Patterson, Evan and Candes, Emmanuel},
 booktitle = {Advances in Neural Information Processing Systems},
 editor = {H. Wallach and H. Larochelle and A. Beygelzimer and F. d\textquotesingle Alch\'{e}-Buc and E. Fox and R. Garnett},
 pages = {},
 publisher = {Curran Associates, Inc.},
 title = {Conformalized Quantile Regression},
 url = {https://proceedings.neurips.cc/paper_files/paper/2019/file/5103c3584b063c431bd1268e9b5e76fb-Paper.pdf},
 volume = {32},
 year = {2019}
}

@article{10.1561/2400000013,
author = {Hazan, Elad},
title = {Introduction to Online Convex Optimization},
year = {2016},
issue_date = {Aug 2016},
publisher = {Now Publishers Inc.},
address = {Hanover, MA, USA},
volume = {2},
number = {3–4},
issn = {2167-3888},
url = {https://doi.org/10.1561/2400000013},
doi = {10.1561/2400000013},
journal = {Found. Trends Optim.},
month = aug,
pages = {157–325},
numpages = {176}
}

@inproceedings{10.5555/3618408.3618508,
author = {Bhatnagar, Aadyot and Wang, Huan and Xiong, Caiming and Bai, Yu},
title = {Improved online conformal prediction via strongly adaptive online learning},
year = {2023},
publisher = {JMLR.org},
booktitle = {Proceedings of the 40th International Conference on Machine Learning},
articleno = {100},
numpages = {27},
location = {Honolulu, Hawaii, USA},
series = {ICML'23}
}

@article{Zhang2024TheBO,
  title={The Benefit of Being Bayesian in Online Conformal Prediction},
  author={Zhiyu Zhang and Zhou Lu and Heng Yang},
  journal={ArXiv},
  year={2024},
  volume={abs/2410.02561},
}

@inproceedings{3692070.3694489,
author = {Zhang, Zhiyu and Bombara, David and Yang, Heng},
title = {Discounted adaptive online learning: towards better regularization},
year = {2024},
publisher = {JMLR.org},
booktitle = {Proceedings of the 41st International Conference on Machine Learning},
articleno = {2419},
numpages = {31},
location = {Vienna, Austria},
series = {ICML'24}
}

@ARTICLE{10121511,
  author={Xu, Chen and Xie, Yao},
  journal={IEEE Transactions on Pattern Analysis and Machine Intelligence}, 
  title={Conformal Prediction for Time Series}, 
  year={2023},
  volume={45},
  number={10},
  pages={11575-11587},
  doi={10.1109/TPAMI.2023.3272339}}

@InProceedings{pmlr-v230-jonkers24a,
  title = 	 {Conformal Predictive Systems Under Covariate Shift},
  author =       {Jonkers, Jef and Van Wallendael, Glenn and Duchateau, Luc and Van Hoecke, Sofie},
  booktitle = 	 {Proceedings of the Thirteenth Symposium on Conformal and Probabilistic Prediction with Applications},
  pages = 	 {406--423},
  year = 	 {2024},
  editor = 	 {Vantini, Simone and Fontana, Matteo and Solari, Aldo and Boström, Henrik and Carlsson, Lars},
  volume = 	 {230},
  series = 	 {Proceedings of Machine Learning Research},
  month = 	 {09--11 Sep},
  publisher =    {PMLR},
  url = 	 {https://proceedings.mlr.press/v230/jonkers24a.html},
  
}

@article{10.1214/23-AOS2276,
author = {Rina Foygel Barber and Emmanuel J. Cand{\`e}s and Aaditya Ramdas and Ryan J. Tibshirani},
title = {{Conformal prediction beyond exchangeability}},
volume = {51},
journal = {The Annals of Statistics},
number = {2},
publisher = {Institute of Mathematical Statistics},
pages = {816 -- 845},
year = {2023},
doi = {10.1214/23-AOS2276},
URL = {https://doi.org/10.1214/23-AOS2276}
}

@article{Lei2014,
author = {Lei, Jing and Wasserman, Larry},
year = {2014},
month = {01},
pages = {},
title = {Distribution-free Prediction Bands for Non-parametric Regression},
volume = {76},
journal = {Journal of the Royal Statistical Society: Series B (Statistical Methodology)},
doi = {10.1111/rssb.12021}
}

@inproceedings{10.5555/3692070.3693062,
author = {Kiyani, Shayan and Pappas, George and Hassani, Hamed},
title = {Conformal prediction with learned features},
year = {2024},
publisher = {JMLR.org},
booktitle = {Proceedings of the 41st International Conference on Machine Learning},
articleno = {992},
numpages = {21},
location = {Vienna, Austria},
series = {ICML'24}
}

@article{pnas2107794118,
author = {Victor Chernozhukov  and Kaspar Wüthrich  and Yinchu Zhu },
title = {Distributional conformal prediction},
journal = {Proceedings of the National Academy of Sciences},
volume = {118},
number = {48},
pages = {e2107794118},
year = {2021},
doi = {10.1073/pnas.2107794118},
URL = {https://www.pnas.org/doi/abs/10.1073/pnas.2107794118},
eprint = {https://www.pnas.org/doi/pdf/10.1073/pnas.2107794118},
}

@inproceedings{
sesia2021conformal,
title={Conformal Prediction using Conditional Histograms},
author={Matteo Sesia and Yaniv Romano},
booktitle={Advances in Neural Information Processing Systems},
editor={A. Beygelzimer and Y. Dauphin and P. Liang and J. Wortman Vaughan},
year={2021},
}

@inproceedings{3618408.3620021,
author = {Xu, Chen and Xie, Yao},
title = {Sequential predictive conformal inference for time series},
year = {2023},
publisher = {JMLR.org},

booktitle = {Proceedings of the 40th International Conference on Machine Learning},
articleno = {1613},
numpages = {21},
location = {Honolulu, Hawaii, USA},
series = {ICML'23}
}

@inproceedings{10.5555/3304222.3304273,
author = {Yu, Bing and Yin, Haoteng and Zhu, Zhanxing},
title = {Spatio-temporal graph convolutional networks: a deep learning framework for traffic forecasting},
year = {2018},
isbn = {9780999241127},
publisher = {AAAI Press},
booktitle = {Proceedings of the 27th International Joint Conference on Artificial Intelligence},
pages = {3634–3640},
numpages = {7},
location = {Stockholm, Sweden},
series = {IJCAI'18}
}

@inproceedings{
li2018diffusion,
title={Diffusion Convolutional Recurrent Neural Network: Data-Driven Traffic Forecasting},
author={Yaguang Li and Rose Yu and Cyrus Shahabi and Yan Liu},
booktitle={International Conference on Learning Representations},
year={2018},
url={https://openreview.net/forum?id=SJiHXGWAZ},
}

@inproceedings{10.1145/3394486.3403118,
author = {Wu, Zonghan and Pan, Shirui and Long, Guodong and Jiang, Jing and Chang, Xiaojun and Zhang, Chengqi},
title = {Connecting the Dots: Multivariate Time Series Forecasting with Graph Neural Networks},
year = {2020},
isbn = {9781450379984},
publisher = {Association for Computing Machinery},
address = {New York, NY, USA},
url = {https://doi.org/10.1145/3394486.3403118},
doi = {10.1145/3394486.3403118},
booktitle = {Proceedings of the 26th ACM SIGKDD International Conference on Knowledge Discovery \& Data Mining},
pages = {753–763},
numpages = {11},
location = {Virtual Event, CA, USA},
series = {KDD '20}
}

@inproceedings{10.5555/3367243.3367303,
author = {Wu, Zonghan and Pan, Shirui and Long, Guodong and Jiang, Jing and Zhang, Chengqi},
title = {Graph wavenet for deep spatial-temporal graph modeling},
year = {2019},
isbn = {9780999241141},
publisher = {AAAI Press},
booktitle = {Proceedings of the 28th International Joint Conference on Artificial Intelligence},
pages = {1907–1913},
numpages = {7},
location = {Macao, China},
series = {IJCAI'19}
}

@book{vershynin_high-dimensional_nodate,
	title = { High-Dimensional Probability: An Introduction with Applications in Data Science},
	author = {Vershynin, Roman},
	publisher= {Cambridge University Press},
year={2022}

}

@article{WANG20211,
title = {Target-oriented robust location–transportation problem with service-level measure},
journal = {Transportation Research Part B: Methodological},
volume = {153},
pages = {1-20},
year = {2021},
issn = {0191-2615},
doi = {https://doi.org/10.1016/j.trb.2021.08.010},
url = {https://www.sciencedirect.com/science/article/pii/S0191261521001612},
author = {Xin Wang and Yong-Hong Kuo and Houcai Shen and Lianmin Zhang},
}

@article{HUANG202390,
title = {Column-and-constraint-generation-based approach to a robust reverse logistic network design for bike sharing},
journal = {Transportation Research Part B: Methodological},
volume = {173},
pages = {90-118},
year = {2023},
issn = {0191-2615},
doi = {https://doi.org/10.1016/j.trb.2023.04.010},
url = {https://www.sciencedirect.com/science/article/pii/S0191261523000735},
author = {Sen Huang and Kanglin Liu and Zhi-Hai Zhang},
}

@article{CHEN2023235,
title = {A target-based optimization model for bike-sharing systems: From the perspective of service efficiency and equity},
journal = {Transportation Research Part B: Methodological},
volume = {167},
pages = {235-260},
year = {2023},
issn = {0191-2615},
doi = {https://doi.org/10.1016/j.trb.2022.12.002},
url = {https://www.sciencedirect.com/science/article/pii/S019126152200203X},
author = {Qingxin Chen and Chenyi Fu and Ning Zhu and Shoufeng Ma and Qiao-Chu He},
}

@inproceedings{Fan2023DishTSAG,
  title={Dish-TS: A General Paradigm for Alleviating Distribution Shift in Time Series Forecasting},
  author={Wei Fan and Pengyang Wang and Dongkun Wang and Dongjie Wang and Yuanchun Zhou and Yanjie Fu},
  booktitle={AAAI Conference on Artificial Intelligence},
  year={2023},

}

@inproceedings{Kim2022ReversibleIN,
  title={Reversible Instance Normalization for Accurate Time-Series Forecasting against Distribution Shift},
  author={Taesung Kim and Jinhee Kim and Yunwon Tae and Cheonbok Park and Jangho Choi and Jaegul Choo},
  booktitle={International Conference on Learning Representations},
  year={2022},
 
}

@inproceedings{
xie2023evolving,
title={Evolving Standardization for Continual Domain Generalization over Temporal Drift},
author={Mixue Xie and Shuang Li and Longhui Yuan and Chi Harold Liu and Zehui Dai},
booktitle={Thirty-seventh Conference on Neural Information Processing Systems},
year={2023},
url={https://openreview.net/forum?id=5hVXbiEGXB}
}

@article{Huang2024,
  author       = {Xiannan Huang and
                  Shuhan Qiu and
                  Yan Cheng and
                  Quan Yuan and
                  Chao Yang},
  title        = {Incorporating Long-term Data in Training Short-term Traffic Prediction
                  Model},
  journal      = {CoRR},
  volume       = {abs/2410.14726},
  year         = {2024},
  url          = {https://doi.org/10.48550/arXiv.2410.14726},
  doi          = {10.48550/ARXIV.2410.14726},
  eprinttype    = {arXiv},
  eprint       = {2410.14726},
  bibsource    = {dblp computer science bibliography, https://dblp.org}
}

@inproceedings{10.5555/3692070.3693235,
author = {Li, Zhonghang and Xia, Lianghao and Xu, Yong and Huang, Chao},
title = {FlashST: A simple and universal prompt-tuning framework for traffic prediction},
year = {2024},
publisher = {JMLR.org},
booktitle = {Proceedings of the 41st International Conference on Machine Learning},
articleno = {1165},
numpages = {11},
location = {Vienna, Austria},
series = {ICML'24}
}

@inproceedings{chen2024calibration,
  title={Calibration of time-series forecasting: Detecting and adapting context-driven distribution shift},
  author={Chen, Mouxiang and Shen, Lefei and Fu, Han and Li, Zhuo and Sun, Jianling and Liu, Chenghao},
  booktitle={Proceedings of the 30th ACM SIGKDD Conference on Knowledge Discovery and Data Mining},
  pages={341--352},
  year={2024}
}

@article{guo2024online,
  title={Online test-time adaptation of spatial-temporal traffic flow forecasting},
  author={Guo, Pengxin and Jin, Pengrong and Li, Ziyue and Bai, Lei and Zhang, Yu},
  journal={arXiv preprint arXiv:2401.04148},
  year={2024}
}

@article{chandra2009predictions,
  title={Predictions of freeway traffic speeds and volumes using vector autoregressive models},
  author={Chandra, Srinivasa Ravi and Al-Deek, Haitham},
  journal={Journal of Intelligent Transportation Systems},
  volume={13},
  number={2},
  pages={53--72},
  year={2009},
  publisher={Taylor \& Francis}
}

@article{okutani1984dynamic,
  title={Dynamic prediction of traffic volume through Kalman filtering theory},
  author={Okutani, Iwao and Stephanedes, Yorgos J},
  journal={Transportation Research Part B: Methodological},
  volume={18},
  number={1},
  pages={1--11},
  year={1984},
  publisher={Elsevier}
}

@article{van1996combining,
  title={Combining Kohonen maps with ARIMA time series models to forecast traffic flow},
  author={Van Der Voort, Mascha and Dougherty, Mark and Watson, Susan},
  journal={Transportation Research Part C: Emerging Technologies},
  volume={4},
  number={5},
  pages={307--318},
  year={1996},
  publisher={Elsevier}
}

@article{luo2023stg4traffic,
  title={Stg4traffic: A survey and benchmark of spatial-temporal graph neural networks for traffic prediction},
  author={Luo, Xunlian and Zhu, Chunjiang and Zhang, Detian and Li, Qing},
  journal={arXiv preprint arXiv:2307.00495},
  year={2023}
}

@inproceedings{li2022lstnet,
  title={LSTnet-GRU-A: Traffic flow forecasting based on attention mechanism},
  author={Li, Mingxin and Zheng, Kunyuan and Zhao, Yuan and Wang, Yongjian and Zhao, Hui and Ding, Lei},
  booktitle={2022 China Automation Congress (CAC)},
  pages={6487--6492},
  year={2022},
  organization={IEEE}
}

@inproceedings{chikkakrishna2022short,
  title={Short-term traffic prediction using fb-prophet and neural-prophet},
  author={ChikkaKrishna, Naveen Kumar and Rachakonda, Pranavi and Tallam, Teja},
  booktitle={2022 IEEE Delhi Section Conference (DELCON)},
  pages={1--4},
  year={2022},
  organization={IEEE}
}

@article{jiang2024spatio,
  title={Spatio-temporal conformal prediction for power outage data},
  author={Jiang, Hanyang and Xie, Yao and Qiu, Feng},
  journal={arXiv preprint arXiv:2411.17099},
  year={2024}
}

@article{jiang2024spatial,
  title={Spatial conformal inference through localized quantile regression},
  author={Jiang, Hanyang and Xie, Yao},
  journal={arXiv preprint arXiv:2412.01098},
  year={2024}
}

@inproceedings{wang2025non,
  title={Non-exchangeable Conformal Prediction for Temporal Graph Neural Networks},
  author={Wang, Tuo and Kang, Jian and Yan, Yujun and Kulkarni, Adithya and Zhou, Dawei},
  booktitle={Proceedings of the 31st ACM SIGKDD Conference on Knowledge Discovery and Data Mining V. 2},
  pages={3031--3042},
  year={2025}
}

@article{song2024similarity,
  title={Similarity-navigated conformal prediction for graph neural networks},
  author={Song, Jianqing and Huang, Jianguo and Jiang, Wenyu and Zhang, Baoming and Li, Shuangjie and Wang, Chongjun},
  journal={Advances in Neural Information Processing Systems},
  volume={37},
  pages={48541--48567},
  year={2024}
}

@inproceedings{xu2024conformal,
  title={Conformal prediction for multi-dimensional time series by ellipsoidal sets},
  author={Xu, Chen and Jiang, Hanyang and Xie, Yao},
  booktitle={International Conference on Machine Learning},
  year={2024}
}

@article{sun2022copula,
  title={Copula conformal prediction for multi-step time series forecasting},
  author={Sun, Sophia and Yu, Rose},
  journal={arXiv preprint arXiv:2212.03281},
  year={2022}
}

@article{2007Introduction,
  title={Introduction to strong mixing conditions},
  author={ Bradley, Richard C. },
  journal={Kendrick Press},
  year={2007},
}

@article{Httel2023DeepEL,
  title={Deep Evidential Learning for Bayesian Quantile Regression},
  author={Frederik Boe H{\"u}ttel and Filipe Rodrigues and Francisco Pereira},
  journal={ArXiv},
  year={2023},
  volume={abs/2308.10650},
}

@article{amini2020deep,
  title={Deep evidential regression},
  author={Amini, Alexander and Schwarting, Wilko and Soleimany, Ava and Rus, Daniela},
  journal={Advances in neural information processing systems},
  volume={33},
  pages={14927--14937},
  year={2020}
}

@inproceedings{meinert2023unreasonable,
  title={The unreasonable effectiveness of deep evidential regression},
  author={Meinert, Nis and Gawlikowski, Jakob and Lavin, Alexander},
  booktitle={Proceedings of the AAAI Conference on Artificial Intelligence},
  volume={37},
  number={8},
  pages={9134--9142},
  year={2023}
}

@article{rodrigues2020beyond,
  title={Beyond expectation: Deep joint mean and quantile regression for spatiotemporal problems},
  author={Rodrigues, Filipe and Pereira, Francisco C},
  journal={IEEE transactions on neural networks and learning systems},
  volume={31},
  number={12},
  pages={5377--5389},
  year={2020},
  publisher={IEEE}
}
\appendix
\newpage

\section{Implementation details}
The pseudo code of our method can be summarized in Algorithm \ref{alg}. 
\begin{algorithm}
\caption{Conformal Traffic Intervals with Adaptation}
\begin{algorithmic}[1]
\Require Training dataset $D_1$, validation dataset $D_2$, confidence level $\alpha$.
\Ensure Obtain test data $(x_{t,i,j}, y_{t,i,j})$ one by one in future $T$ steps.

\State Train quantile prediction model $M$ using $D_1$.

\For{$i \in [1, n]$}
    \For{$j \in \{1, 2\}$}
        \State Obtain error set $E_{1,i,j}$ in $D_2$ using Equation\ref{eq_e}.
    \EndFor
    \State Initialize $\alpha_{1,i} = \alpha$, $v_{1,i} = 0$.
\EndFor

\For{$t \in [1, T]$}
    \For{$i \in [1, n]$}
        \For{$j \in \{1, 2\}$}
            \State Observe $x_{t,i,j}$.
            \State Obtain predicted quantiles $y_{t,i,j,\alpha/2}$, $y_{t,i,j,1-\alpha/2}$ using model $M$.
            \State Output predicted interval:
            $$
            C_{1-\alpha}(x_{t,i,j}) = \left[y_{t,i,j,\alpha/2} - Q_{1-\alpha_{t,i}}(E_{t,i,j}), \; y_{t,i,j,1-\alpha/2} + Q_{1-\alpha_{t,i}}(E_{t,i,j})\right].
            $$
            \State Observe $y_{t,i,j}$.
            \State Calculate $e_{t,i,j}$ using Equation\ref{eq:coverage_error}.
            \State Obtain $E_{t+1,i,j}$ by adding $e_{t,i,j}$ to $E_{t,i,j}$ and deleting the oldest element of it.
        \EndFor
        \State Calculate error $err_{t,i}$ using Equation\ref{eq_e}.
        \State Obtain $\alpha_{t+1,i}$, $v_{t+1,i}$ using Equation\ref{eq_v} and Equation\ref{eq_al}.
    \EndFor
\EndFor
\end{algorithmic}
\label{alg}
\end{algorithm}
\subsection{Base prediction models}
\begin{itemize}
    \item \textbf{STGCN}: \cite{10.5555/3304222.3304273} (Spatial Temporal Graph Convolutional Network) consists of multiple spatial-temporal convolution blocks which integrate graph convolutions to extract spatial features and gated temporal convolutions to capture temporal dynamics, allowing it to effectively process spatial-temporal data.
    \item \textbf{DCRNN} \cite{li2018diffusion} (Diffusion Convolutional Recurrent Neural Network) models traffic flow as a diffusion process over a directed graph and captures spatial and temporal dependencies through diffusion convolutions on graphs and an encoder-decoder architecture. This model leverages bidirectional random walks on graphs to capture spatial correlations and uses recurrent neural networks like GRUs to model temporal sequences.
    \item \textbf{MTGNN} \cite{10.1145/3394486.3403118} (Multivariate Time Series Forecasting with Graph Neural Networks) automatically extracts relationships between regions (or grids) via a graph learning module. With its mix-hop propagation layers and dilated inception layers, MTGNN captures complex spatial and temporal dependencies while addressing challenges like unknown graph structures and joint optimization of graph structure and network parameters.
    \item \textbf{GWNET} \cite{10.5555/3367243.3367303} (Graph WaveNet) extends the concept of Wavenet to graph-structured data by incorporating adaptive adjacency matrices learned during training, thus overcoming limitations posed by predefined graph structures. And it employs dilated causal convolutions along with graph convolutions to efficiently capture long-range dependencies.
\end{itemize}
The implementation of these 4 models is based on \cite{10.1145/3474717.3483923}, and the hyper-parameters, such as hidden dimension of DCRNN, number of layers in STGCN, are the default value in \cite{10.1145/3474717.3483923}. For all models, we transfer it from point prediction version to quantiles prediction version by adding an additional prediction head and training it with quantile loss. Before training starts, the dataset is normalized using z-score standardization. These models are trained with Adam for 100 epochs with initial learning rate 0.005. If the validation loss did not decrease for 5 consecutive epochs, the learning rate will be halved. And the training will be stopped if the loss in validation set dose not decrease for 10 epochs continuously.
\subsection{Details of baseline methods}
\begin{itemize}
    \item \textbf{Bootstrap:} It is a technique to estimate statistics by sampling a dataset with replacement. In our experiments, we train 20 models with randomly selected training samples and the variance of outputs of difference models ($\sigma^2$) is considered as the variance of prediction. Then if the mean of predictions of all models is $y$, the confidence interval is $[y-1.645\sigma,y+1.645\sigma]$.   
    \item \textbf{MC Dropout}. It is a Bayesian approximation technique that leverages dropout to estimate model uncertainty. In our experiments, we set dropout rate as 0.3, the same as \cite{10304591}. And during deployment, we use models to predict traffic value 20 times and obtain the variance ($\sigma^2$) and means of these predictions ($y$). Then the confidence interval is $[y-1.645\sigma,y+1.645\sigma]$.
    \item \textbf{Directly Minimizing Mean Interval Score} is training models with the following loss:
\begin{align}
MIS &= \frac{1}{2nT} \sum_{t=1}^{T} \sum_{i=1}^{n} \sum_{j=1}^{2} \biggl[ 
    (up_{tij} - low_{tij}) \nonumber \\
    &\quad + \frac{2}{\alpha}(low_{tij} - y_{t,i,j}) \mathbb{I}(y_{t,i,j} < low_{tij}) \nonumber \\
    &\quad + \frac{2}{\alpha}(y_{t,i,j} - up_{t,i,j}) \mathbb{I}(y_{t,i,j} > up_{tij})
\biggr]
\end{align}
\item \textbf{DESQRUQ} \cite{10495765}: The main idea of it is training multiple quantile prediction models with different hyperparameters and ensemble these results. Bayesian optimization and Gaussian copula are used to find better hypeparameter. The search space of hype-parameters in our experiments is: learning rate [0.0001, 0.0005,0.001,0.005,0.01], batch size [8, 16, . . . , 256], number of layers for the encoder [1, 2, 3, 4, 5], number of training epoch [20,21,. . . ,100].
\item \textbf{ProbGNN} \cite{WangWZKZ24}: This method regards the traffic demand as a distribution to consider aleatoric uncertainty and use ensembles to account for epistemic uncertainty. In our experiments, Gaussian distribution is considered as data distribution, the same as \cite{WangWZKZ24}. Besides, we train 20 models with different initialization and select the top 5 models by validation set loss to create an ensemble model, which is also the same as \cite{WangWZKZ24}.
\item \textbf{UATGCN} \cite{10458002}: This method uses Monte Carlo dropout and predictive variances to estimate epistemic and aleatoric uncertainty. We set dropout rate as 0.2, the same as \cite{10458002}.
\item \textbf{QUANTARFFIC} \cite{10304591}: A quantile repression model is trained and validation set is used to adjusts the quantile prediction of each region. The settings in our experiments are the same as in the original paper.
\end{itemize}
\newpage
\section{Proofs}
\label{proofs}
\subsection{Proofs of Theorem \ref{Tho:Avg_cov}}
To prove the theorem of average coverage, we need some lemmas first.

\begin{lemma}
For any region $i$ and any time $t$:
$$ -\frac{\gamma_1}{\alpha\sqrt{(1-\beta)k}+\epsilon} \leq \alpha_{t,i} \leq 1 + \frac{\gamma_1}{\alpha\sqrt{(1-\beta)k}+\epsilon} $$ where $k=\min\{1,\frac{(0.5-\alpha)^2}{\alpha^2},\frac{(1-\alpha)^2}{\alpha^2}\}$ is a constant.
\label{lemma1}
\end{lemma}

\begin{proof}
We prove it by induction:

1. For $t=1$, $\alpha_{1,i} = 1-\alpha$, which satisfies $-\frac{\gamma_1}{\alpha\sqrt{(1-\beta)k}+\epsilon} \leq \alpha_{1,i} \leq 1 + \frac{\gamma_1}{\alpha\sqrt{(1-\beta)k}+\epsilon}$.

2. Assume for time $t$:
$$ -\frac{\gamma_1}{\alpha\sqrt{(1-\beta)k}+\epsilon} \leq \alpha_{t,i} \leq 1 +\frac{\gamma_1}{\alpha\sqrt{(1-\beta)k}+\epsilon}$$

We now prove the statement for $t+1$:

\begin{itemize}

    \item If $-\frac{\gamma_1}{\alpha\sqrt{(1-\beta)k}+\epsilon} \leq \alpha_{t,i} <0$:
    
    Then $Q_{1-\alpha_{t,i}}(E_{t,i})=\infty$ and $C_{t,i} = \mathbb{R}$, as a result, $\mathbb{P}(y_{t,i} \in C_{t,i}) = 1$. Thus according to Equation \ref{eq:alpha_update}:
    \begin{equation}
        \alpha_{t,i+1} = \gamma_{t,i}\left(\alpha - \frac{1}{2}\sum_{j=1}^2 \mathbb{I}(y_{t,i,j} \notin [\mathrm{low}_{t,i,j}, \mathrm{up}_{t,i,j}])\right) + \alpha_{t,i}=\gamma_{t,i}\alpha+\alpha_{t,i}>\alpha_{t,i}
    \end{equation}
   
    Then according to the assumption:
    $$\alpha_{t,i}\geq -\frac{\gamma_1}{\alpha\sqrt{(1-\beta)k}+\epsilon}$$
    we have $\alpha_{t,i+1}\geq -\frac{\gamma_1}{\alpha\sqrt{(1-\beta)k}+\epsilon}$.
    
    Then we only need to show:
    \begin{equation}
        \alpha_{t,i+1} =\gamma_{t,i}\alpha+\alpha_{t,i}\leq 1+\frac{\gamma_1}{\alpha\sqrt{(1-\beta)k}+\epsilon}
        \label{eq__a_r1}
    \end{equation}

    according to Equation \ref{eq_v}, $v_{t,i+1} = \beta v_{t,i} + (1-\beta)g_t^2$ , where $$g_t^2 = \left(\alpha - \frac{1}{2}\sum_{j=1}^2 \mathbb{I}(y_{t,i,j} \notin [\mathrm{low}_{t,i,j}, \mathrm{up}_{t,i,j}])\right)^2$$
    The possible values for $g_t^2$ are $\{\alpha^2,(0.5-\alpha)^2,(1-\alpha)^2\}$, then we have $g_t^2\geq k\alpha^2$. Therefore,
    \begin{equation}
        v_{t,i+1} = \beta v_{t,i} + (1-\beta)g_t^2\geq (1-\beta)g_t^2 \geq (1-\beta)k\alpha^2
    \end{equation}

    which implies:
    $$ \gamma_{t,i} = \frac{\gamma_1}{\sqrt{v_{t,i} }+\epsilon} \leq \frac{\gamma_1}{\alpha\sqrt{(1-\beta)k}+\epsilon} $$
    Hence $$\alpha_{t,i+1} <0+\gamma_{t,i}\alpha\leq\alpha \frac{\gamma_1}{\alpha\sqrt{(1-\beta)k}+\epsilon}<1+\frac{\gamma_1}{\alpha\sqrt{(1-\beta)k+\epsilon}}$$
    Therefore, inequality \ref{eq__a_r1} holds.

    \item If $1 < \alpha_{t,i} \leq 1+\frac{\gamma_1}{\alpha\sqrt{(1-\beta)k}+\epsilon}$, then $C_{t,i} = \emptyset$ and $\mathbb{P}(y_{t,i} \in C_{t,i}) = 0$. Thus
    \begin{equation}
        \alpha_{t+1,i} = (\alpha - 1)\gamma_{t,i} + \alpha_{t,i}\geq(\alpha-1)\frac{\gamma_1}{\alpha\sqrt{(1-\beta)k}+\epsilon}+1>-\frac{\gamma_1}{\alpha\sqrt{(1-\beta)k}+\epsilon}
    \end{equation}
   
    and
    \begin{equation}
        \alpha_{t+1,i} = (\alpha - 1)\gamma_{t,i} + \alpha_{t,i}<\alpha_{t,i}\leq 1+\frac{\gamma_1}{\alpha\sqrt{(1-\beta)k}+\epsilon}
    \end{equation}
   
    remains in $\left[-\frac{\gamma_1}{\alpha\sqrt{(1-\beta)k}}, 1 + \frac{\gamma_1}{\alpha\sqrt{(1-\beta)k}}\right]$.

    \item If $0 < \alpha_{t,i} < 1$:
    $$ \alpha_{t,i+1} = \gamma_{t,i}\left(\alpha - \frac{1}{2}\sum_{j=1}^2 \mathbb{I}(y_{t,i,j} \notin [\mathrm{low}_{t,i,j}, \mathrm{up}_{t,i,j}])\right) + \alpha_{t,i} $$
    Since $\gamma_{t,i}\left(\alpha - \frac{1}{2}\sum_{j=1}^2 \mathbb{I}(\cdot)\right) \in \left[-\frac{\gamma_1}{\alpha\sqrt{(1-\beta)k}+\epsilon},\frac{\gamma_1}{\alpha\sqrt{(1-\beta)k}+\epsilon}\right]$, we have $$\alpha_{t+1,i} \in \left[-\frac{\gamma_1}{\alpha\sqrt{(1-\beta)k}+\epsilon}, 1 + \frac{\gamma_1}{\alpha\sqrt{(1-\beta)k}+\epsilon}\right]$$
\end{itemize}

Combining these 3 cases, we conclude that for all $t$:
$$ -\frac{\gamma_1}{\alpha\sqrt{(1-\beta)k}+\epsilon} \leq \alpha_{t,i} \leq 1 + \frac{\gamma_1}{\alpha\sqrt{(1-\beta)k}+\epsilon} $$ 
\end{proof}
\begin{lemma}
For any region $i$ and any time $t$:
$$
v_{t,i}<1
$$
\label{lemma:bound_v}
\end{lemma}
\begin{proof}
    because of Equation \ref{eq_v}:
    \begin{align*}
        v_{t,i}&=\beta v_{t-1,i}+(1-\beta)(err_{t,i}-\alpha)^2\  
        \\&\text{substitute $v_{t-1,i}$ with $v_{t-2,i}$} \\
        &=\beta(\beta v_{t-2,i}+(1-\beta)(err_{t-1,i}-\alpha)^2)+(1-\beta)(err_{t,i}-\alpha)^2\\
        &=\beta^2v_{t-2,i}+\beta(1-\beta)(err_{t-1,i}-\alpha)^2+(1-\beta)(err_{t,i}-\alpha)^2\\
        &\text{substitute $v_{t-2,i}$ with $v_{t-3,i}$}\\
        &=\beta^3v_{t-3,i}+\beta^2(1-\beta)(err_{t-2,i}-\alpha)^2+\beta(1-\beta)(err_{t-1,i}-\alpha)^2+(1-\beta)(err_{t,i}-\alpha)^2\\
        &\text{substitute $v_{t-3,i}$ with $v_{t-4,i}$}\\
        &......\\
        &=\beta^{t-1}v_{1,i}+(1-\beta)\sum_{r=0}^{t-2}{\beta^r(err_{t-r,i}-\alpha)^2}\\
        &< (1-\beta)\sum_{r=0}^{t-2}{\beta^r}\\
        &=(1-\beta)\frac{1-\beta^{t-1}}{1-\beta}<1
    \end{align*}

\end{proof}
We rewrite Theorem \ref{Tho:Avg_cov} as follows:
\setcounter{theorem}{0}
\begin{theorem}
For any $\alpha \in (0,1)$, we have:
\begin{equation}
\left|\frac{1}{2nT}\sum_{t=1}^T \sum_{i=1}^n \sum_{j=1}^2 \mathbb{I}(y_{t,i,j} \in [\mathit{low}_{t,i,j}, \mathit{up}_{t,i,j}]) - (1-\alpha)\right| \leq \frac{c}{T}
\end{equation}

\end{theorem}

\begin{proof}
For any region $i$, define $\Delta_i$ with components:
$$
[\frac{1}{\gamma_{1,i}}, \frac{1}{\gamma_{2,i}} - \frac{1}{\gamma_{1,i}},\frac{1}{\gamma_{3,i}} - \frac{1}{\gamma_{2,i}},...]
$$

and let $\|\Delta_i\|_1 = \sum_j |\Delta_{i,j}|$ be the $\ell_1$ norm of $\Delta_i$.

We can rewrite the coverage difference as:
\begin{align}
&\left|\frac{1}{2T}\sum_{t=1}^T \sum_{j=1}^2 \mathbb{I}(y_{t,i,j} \in [\mathit{low}_{t,i,j}, \mathit{up}_{t,i,j}]) - (1-\alpha)\right| \nonumber \\
&= \left|\frac{1}{2T}\sum_{t=1}^T \sum_{j=1}^2 \mathbb{I}(y_{t,i,j} \notin [\mathit{low}_{t,i,j}, \mathit{up}_{t,i,j}]) - \alpha\right| \nonumber \\
&= \left|\frac{1}{2T}\sum_{t=1}^T \left(\sum_{r=1}^t \Delta_{i,r}\right) \gamma_{t,i} \sum_{j=1}^2 \mathbb{I}(y_{t,i,j} \notin [\mathit{low}_{t,i,j}, \mathit{up}_{t,i,j}]) - \alpha\right| \nonumber \\
&= \left|\frac{1}{T}\sum_{r=1}^T \Delta_{i,r} \sum_{t=r}^T \left(\frac{1}{2}\sum_{j=1}^2 \gamma_{t,i} \mathbb{I}(y_{t,i,j} \notin [\mathit{low}_{t,i,j}, \mathit{up}_{t,i,j}]) - \alpha\right)\right| 
\label{eq:b2}
\end{align}
because the update rule of $\alpha_{r,i}$ is:
$$
\alpha_{(r+1),i}=\gamma_{r,i} (\alpha-\frac{1}{2} \sum_{j=1}^2\mathbb{I}(y_{t,i,j}\notin[low_{t,i,j},up_{t,i,j} ]) )+\alpha_{r.i}
$$
Then \ref{eq:b2} becomes:
\begin{equation}
\left|\frac{1}{T}\sum_{r=1}^T \Delta_{i,r} (\alpha_{T,i} - \alpha_{i,r})\right| \leq \frac{1}{T} \max_r |\alpha_{T,i} - \alpha_{i,r}| \sum_{r=1}^T |\Delta_{i,r}| 
\label{eq:bound_cov}
\end{equation}

By Lemma \ref{lemma1}, we have $-\frac{\gamma_1}{\alpha\sqrt{(1-\beta)k}+\epsilon} \leq 1+\alpha_{i,r} \leq \frac{\gamma_1}{\alpha\sqrt{(1-\beta)k}+\epsilon}$, thus:
\begin{equation}
\max_r |\alpha_{iT} - \alpha_{ir}| \leq 1 + 2\frac{\gamma_1}{\alpha\sqrt{(1-\beta)k}+\epsilon}
\label{eq_alpha_bound}
\end{equation}

For the $\Delta_i$ terms, since $\gamma_{t,i} = \frac{\gamma_1}{\sqrt{v_{t,i}}+\epsilon }$, we have:
\begin{equation}
\frac{1}{\gamma_{t+1,i}} - \frac{1}{\gamma_{t,i}} = \frac{\sqrt{v_{t,i}} - \sqrt{v_{t+1,i}}}{\gamma_1}
\end{equation}
and thus:
\begin{equation}
\sum_{r=1}^T |\Delta_{i,r}|=\|\Delta_i\|_1 = \frac{\sqrt{v_{T,i}} - \sqrt{v_{1,i}}}{\gamma_1}= \frac{\sqrt{v_{T,i}} }{\gamma_1}<\frac{1}{\gamma_1}  \label{eq:bound_v1}
\end{equation}
The last inequality is from Lemma \ref{lemma:bound_v}. Combining \ref{eq:bound_cov}, \ref{eq_alpha_bound} and \ref{eq:bound_v1} gives:
\begin{equation}
\left|\frac{1}{2T}\sum_{t=1}^T \sum_{j=1}^2 \mathbb{I}(y_{t,i,j} \in [\mathit{low}_{t,i,j}, \mathit{up}_{t,i,j}]) - (1-\alpha)\right| \leq \frac{1}{T}(\frac{1}{\gamma_1}+\frac{2}{\alpha\sqrt{(1-\beta)k}+\epsilon})
\label{eq:each_region}
\end{equation}

The theorem follows by aggregating over all regions $i$, where:
$$c=\frac{1}{\gamma_1}+\frac{2}{\alpha\sqrt{(1-\beta)k}+\epsilon}$$
\end{proof}
\subsection{Proofs of Theorem \ref{Tho:min_cov}}

\begin{lemma}[Hoeffding's Inequality (Theorem 2.2.6 in \cite{vershynin_high-dimensional_nodate}]
Let $X_1,X_2,\ldots,X_n$ be independent random variables with $a_i \leq X_i \leq b_i$. Then for any $\epsilon > 0$:
\begin{equation}
P\left(\left|\sum_{i=1}^n X_i - E\sum_{i=1}^n X_i\right| \geq \epsilon\right) \leq 2\exp\left(-\frac{2\epsilon^2}{\sum_{i=1}^n (b_i-a_i)^2}\right)
\end{equation}
\label{lemma_hoe}
\end{lemma}

\begin{lemma}[Sub-Gaussian Properties (Proposition 2.5.2 in \cite{vershynin_high-dimensional_nodate})]
For a zero-mean random variable $X$, the followings are equivalent:
\begin{enumerate}
\item Tail bound: $\forall t>0,\; P(|X| \geq t) \leq 2\exp\left(-\frac{t^2}{2\sigma^2}\right)$
\item Moment generating function: $\forall \lambda \in \mathbb{R},\; E[e^{\lambda X}] \leq \exp\left(\frac{\lambda^2\sigma^2}{2}\right)$
\item Moment bounds: $\forall p \geq 1,\; \|X\|_{L^p} = (E|X|^p)^{1/p} \leq K\sigma\sqrt{p}$
\end{enumerate}
where $K$ is an absolute constant.
\label{lemma_subg}
\end{lemma}

\begin{lemma}[Concentration of Averages]
For random variables $x_1,x_2,...,x_n$ with $P(|x_i| \geq t) \leq 2\exp\left(-\frac{t^2}{2\sigma_i^2}\right)$:
\begin{equation}
P\left(\left|\frac{1}{n}\sum_{i=1}^n X_i\right| \geq t\right) \leq 2\exp\left(-\frac{t^2}{2\max_i \sigma_i^2}\right)
\end{equation}
\label{lemma_avg}
\end{lemma}

\begin{proof}
Using the moment bound from Lemma \ref{lemma_subg}, for any $p\geq1$:
$$
\left\|\frac{1}{n}\sum_{i=1}^n X_i\right\|_{L^p} \leq  \frac{1}{n}\sum_{i=1}^n \|X_i\|_{L^p} \leq \frac{K}{n}\sqrt{p}\sum_{i=1}^n \sigma_i \leq \sqrt{p}K\max_i \sigma_i
$$
The first inequality is because of Minkowski’s inequality and the second inequality is because of $\textit{3.}$ in Lemma \ref{lemma_subg}. 
\end{proof}

\begin{theorem}[Coverage guarantee for the worst region]
If for any region $i$, index $j$ and time step $t,t'$, such that $|t'-t|\geq K$:

\begin{equation}
   err_{t,i,j}\perp err_{t',i,j}
    \label{assu_in}
\end{equation}

we have:
\begin{equation}
\min_i \frac{1}{2T}\sum_{t=1}^T \sum_{j=1}^2 \mathbb{I}(y_{t,i,j} \in [\mathit{low}_{t,i,j}, \mathit{up}_{t,i,j}]) \geq 1 - \alpha - \frac{c_1}{T} - \sqrt{\frac{c_2K\log n}{T}}
\end{equation}

\end{theorem}

\begin{proof}
Define $M_{t,i}$ as:
$$
M_{t,i} = \frac{1}{2}\sum_{j=1}^2 \mathbb{I}(y_{t,i,j} \in [\mathit{low}_{t,i,j}, \mathit{up}_{t,i,j}])
$$
Then we have:
\begin{equation}
    \min_i \sum_{t=1}^T M_{t,i}=\min_i \frac{1}{2T}\sum_{t=1}^T \sum_{j=1}^2 \mathbb{I}(y_{t,i,j} \in [\mathit{low}_{t,i,j}, \mathit{up}_{t,i,j}])
    \label{eq_m_mean}
\end{equation}

Note $M_{t,i} \in [-1, 1]$. Partition time steps into $K$ sets:
$$
S_k = \{i \mid i = nK + k, n \in \mathbb{N}, i \leq T\}, \quad k = 1,\ldots,K
$$

For each $k$, $\{M_{t,i}\}_{t \in S_k}$ are independent because of our assumption \ref{assu_in}. By Lemma \ref{lemma_hoe}:
$$
P\left(\left|\sum_{t \in S_k} M_{t,i} - E\sum_{t \in S_k} M_{t,i}\right| \geq \epsilon\right) \leq 2\exp\left(-\frac{4\epsilon^2}{|S_k|}\right)
$$

Without loss of generality, we assume $T$ is divisible by $K$, we have $|S_k| = T/K$. According to Lemma \ref{lemma_avg}:
\begin{equation}
    P\left(\frac{1}{T}\left|\sum_{t=1}^T M_{t,i} - E\sum_{t=1}^T M_{t,i}\right| \geq \epsilon\right) \leq 2\exp\left(-\frac{4\epsilon^2 T}{K}\right)
    \label{eq_con_m}
\end{equation}

Because of \ref{eq:each_region}, we can find $c_1$, such that, $E\left[\frac{1}{T}\sum_{t=1}^T M_{t,i}\right] \geq 1 - \alpha - \frac{c_1}{T}$. We define:
$$
u_i = 1 - \alpha - \frac{c_1}{T} - \frac{1}{T}\sum_{t=1}^T M_{t,i}
$$
Then $Eu_i\leq 0$
\begin{equation}
     P(|u_i-Eu_i|\geq \epsilon) \leq 2\exp\left(-\frac{4\epsilon^2 T}{K}\right)
    \label{eq_con_u}
\end{equation}
The second $\leq$ is derived from Equation \ref{eq_con_m}.
Besides, we have:
\begin{equation}
    \max_{i} u_i=1-\alpha-\frac{c_1}{T}-\min_i\sum_{t=1}^T M_{t,i}
    \label{mean_u}
\end{equation}

If we define $v_i=u_i-Eu_i$, then we bound $E(\max_i{v_i})$ as follows:
\begin{align*}
&\exp\left(\lambda E[\max_i v_i]\right) \\
&\quad \leq E\left[\exp\left(\lambda \max_i v_i\right)\right] \quad \text{(Jensen's inequality)} \\
&\quad = E\left[\max_i e^{\lambda v_i}\right] \quad \leq E\left[\sum_{i=1}^n e^{\lambda v_i}\right] \quad \text{(since } \max_i e^{\lambda v_i} \leq \sum_i e^{\lambda v_i}) \\
&\quad = \sum_{i=1}^n E\left[e^{\lambda v_i}\right] \quad \leq n \exp\left(\frac{\lambda^2 K}{8T}\right) \quad \text{(by Equation \ref{eq_con_u} and Lemma \ref{lemma_subg})} \\
&\quad = \exp\left(\log n + \frac{\lambda^2 K}{16T}\right)
\end{align*}

\text{Taking logs on both sides:}
\begin{equation}
\lambda E[\max_i v_i] \leq \log n + \frac{\lambda^2 K}{16T}
\end{equation}
Therefore, we obtain:
\begin{equation}
    E\left[\max_{i} v_i\right] \leq \frac{\log n}{\lambda} + \frac{K\lambda}{16T}\leq\sqrt{\frac{K\log n}{4T}}
    \label{eq_bound_v}
\end{equation}

Then recall \ref{eq_m_mean} and \ref{mean_u}:
\begin{align*}
    E\min_i \frac{1}{2T}\sum_{t=1}^T 
    &\sum_{j=1}^2 \mathbb{I}(y_{t,i,j} \in [\mathit{low}_{t,i,j}, \mathit{up}_{t,i,j}])\\&=E\min_iM_i \quad\text{ (By \ref{eq_m_mean})}\\
    &=1-\alpha-\frac{c_1}{T}-E\max_iu_i \quad\text{ (By \ref{mean_u})}\\
    &\geq1-\alpha-\frac{c_1}{T}- E\max_iv_i\\
    &\geq1-\alpha-\frac{c_1}{T}-\sqrt{\frac{c_2K\log n}{T}} \quad\text{ (By \ref{eq_bound_v})}
\end{align*}

\end{proof}
\newpage
\section{Full result}
\subsection{Full result of all experiments}
\label{A:full_res}
We report the results of all prediction models in the following 4 Tables (Table \ref{tab:nycbike_all}, \ref{tab:nyctaxi_all}, \ref{tab:chibike_all}, \ref{tab:chibike_all}), each one represents the results for one dataset. 

It can be observed that our method achieves the best prediction results 15, 20, 13, and 11 times across four datasets, respectively. In cases where our method fails to achieve the best prediction result, it typically obtains the second-best prediction result. This undoubtedly demonstrates the competitiveness of our approach.
\begin{table}[!h]
\centering
\setlength{\tabcolsep}{1.pt}
\renewcommand{\arraystretch}{1.}
\fontsize{5.}{7.5}\selectfont 
\begin{tabular}{@{}c|cccccccccccccccc@{}}
\toprule
Time                        & Base model              & metric & QR                                                  & MCD & boostrop & MIS &DESQRUQ                                                    & UAT-GCN                                                    & Prob-GNN                                                   & DeepJMQR & QuanTraffic                                                & CP                                                   & ACI                                                      & QCP                                                  & DtACI                                                      & CONTINA                                                    \\ \midrule
                            &                         & cov    & 88.0\%                                                        & 55.4\%       & 32.4\%            & 87.8\%       & \cellcolor[HTML]{D9E1F4}91.8\%                                      & 90.8\%                                                              & \cellcolor[HTML]{D9E1F4}91.9\%                                      & 89.67\%           & \cellcolor[HTML]{D9E1F4}91.1\%                                      & 89.1\%                                                        & \cellcolor[HTML]{D9E1F4}89.8\%                                      & \cellcolor[HTML]{D9E1F4}89.5\%                                & \cellcolor[HTML]{D9E1F4}89.9\%                                      & \cellcolor[HTML]{D9E1F4}89.7\%                                      \\
\bottomrule
\end{tabular}
\caption{Results of NYCbike dataset}
\label{tab:nycbike_all}
\end{table}
\newpage

\begin{table}[!h]
\centering
\setlength{\tabcolsep}{1.pt}
\renewcommand{\arraystretch}{1.}
\fontsize{5.}{7.5}\selectfont 

\begin{tabular}{@{}c|cccccccccccccccc@{}}
\toprule
Time                        & Base model              & metric               & \multicolumn{1}{c}{QR} & \multicolumn{1}{c}{MCD} & \multicolumn{1}{c}{boostrop} & \multicolumn{1}{c}{MIS} & \multicolumn{1}{c}{DESQRUQ}                                         & \multicolumn{1}{c}{UAT-GCN} & \multicolumn{1}{c}{Prob-GNN}   & DeepJMQR             & \multicolumn{1}{c}{QuanTraffic} & \multicolumn{1}{c}{CP}         & \multicolumn{1}{c}{ACI}                                             & \multicolumn{1}{c}{QCP} & \multicolumn{1}{c}{DtACI}                                           & \multicolumn{1}{c}{CONTINA}                                   \\ \midrule
                            &                         & cov                  & 88.2\%                 & 60.0\%                  & 33.9\%                       & 88.0\%                  & 92.5\%                                                              & 90.2\%                      & 92.6\%                         & 89.1\%               & 90.7\%                          & 90.5\%                         & \cellcolor[HTML]{D9E1F4}90.0\%                                      & 90.1\%                  & \cellcolor[HTML]{D9E1F4}90.1\%                                      & \cellcolor[HTML]{D9E1F4}89.7\%                                \\
                            &                         & length               & 0.223                  & 0.168                   & 0.072                        & 0.233                   & 0.262                                                               & 0.285                       & 0.272                          & 0.237                & 0.241                           & 0.316                          & \cellcolor[HTML]{D9E1F4}{\color[HTML]{7030A0} {\ul \textbf{0.303}}} & 0.226                   & \cellcolor[HTML]{D9E1F4}0.305                                       & \cellcolor[HTML]{D9E1F4}{\color[HTML]{FF0000} \textbf{0.238}} \\
\bottomrule
\end{tabular}
\caption{Results of NYCtaxi dataset}
\label{tab:nyctaxi_all}
\end{table}
\clearpage
\newpage
\begin{table}[!h]
\centering
\setlength{\tabcolsep}{1.pt}
\renewcommand{\arraystretch}{1.}
\fontsize{5.}{7.5}\selectfont 
\begin{tabular}{@{}c|cccccccccccccccc@{}}
\toprule
Time                        & Base model              & metric & QR                                                            & MCD    & boostrop & MIS                                                                 & DESQRUQ                                                             & UAT-GCN                                                             & Prob-GNN                                                            & DeepJMQR                                                      & QuanTraffic & CP                             & ACI                            & QCP                                                           & DtACI                          & CONTINA                                                             \\ \midrule
                            &                         & cov    & 88.0\%                                                        & 32.0\% & 22.5\%   & 86.2\%                                                              & \cellcolor[HTML]{D9E1F4}93.3\%                                      & 90.8\%                                                              & \cellcolor[HTML]{D9E1F4}92.3\%                                      & \cellcolor[HTML]{D9E1F4}91.2\%                                & 87.7\%      & \cellcolor[HTML]{D9E1F4}90.0\% & \cellcolor[HTML]{D9E1F4}90.2\% & 90.0\%                                                        & \cellcolor[HTML]{D9E1F4}90.2\% & \cellcolor[HTML]{D9E1F4}89.7\%                                      \\
\bottomrule
\end{tabular}
\caption{Results of CHIbike dataset}
\label{tab:chibike_all}
\end{table}
\newpage

\begin{table}[!h]
\centering

\setlength{\tabcolsep}{1.pt}
\renewcommand{\arraystretch}{1.}
\fontsize{5.}{7.5}\selectfont 

\begin{tabular}{@{}c|cccccccccccccccc@{}}
\toprule
Time                        & Base model              & metric & QR                                     & MCD & boostrop & MIS                                    & DESQRUQ                                & UAT-GCN & Prob-GNN & DeepJMQR                                                   &QuanTraffic & CP & ACI & QCP                            & DtACI                                  & CONTINA                               \\ \midrule
                            &                         & cov    & 91.6\%                                                              & 45.6\%                           & 44.6\%                                & 91.5\%                                                              & \cellcolor[HTML]{D9E1F4}93.0\%                                      & \cellcolor[HTML]{D9E1F4}93.1\%       & \cellcolor[HTML]{D9E1F4}94.0\%        & 90.8\%                                                              & \cellcolor[HTML]{D9E1F4}89.8\%           & \cellcolor[HTML]{D9E1F4}92.1\%  & \cellcolor[HTML]{D9E1F4}89.9\%   & 91.3\%                                                        & \cellcolor[HTML]{D9E1F4}90.4\%                                      & \cellcolor[HTML]{D9E1F4}89.6\%                                      \\
\bottomrule
\end{tabular}
\caption{Results of CHItaxi dataset}
\label{tab:chitaxi_all}
\end{table}
\newpage
\subsection{ Full result of sensitive analysis}
\label{section:sa_full}
We represent the result of sensitive analysis for NYCtaxi, CHIbike and CHItaxi datasets in the following Figure \ref{fig:1}, \ref{fig:2}, \ref{fig:3}.
\begin{figure}[!h]
    \centering
    \includegraphics[width=0.8\linewidth]{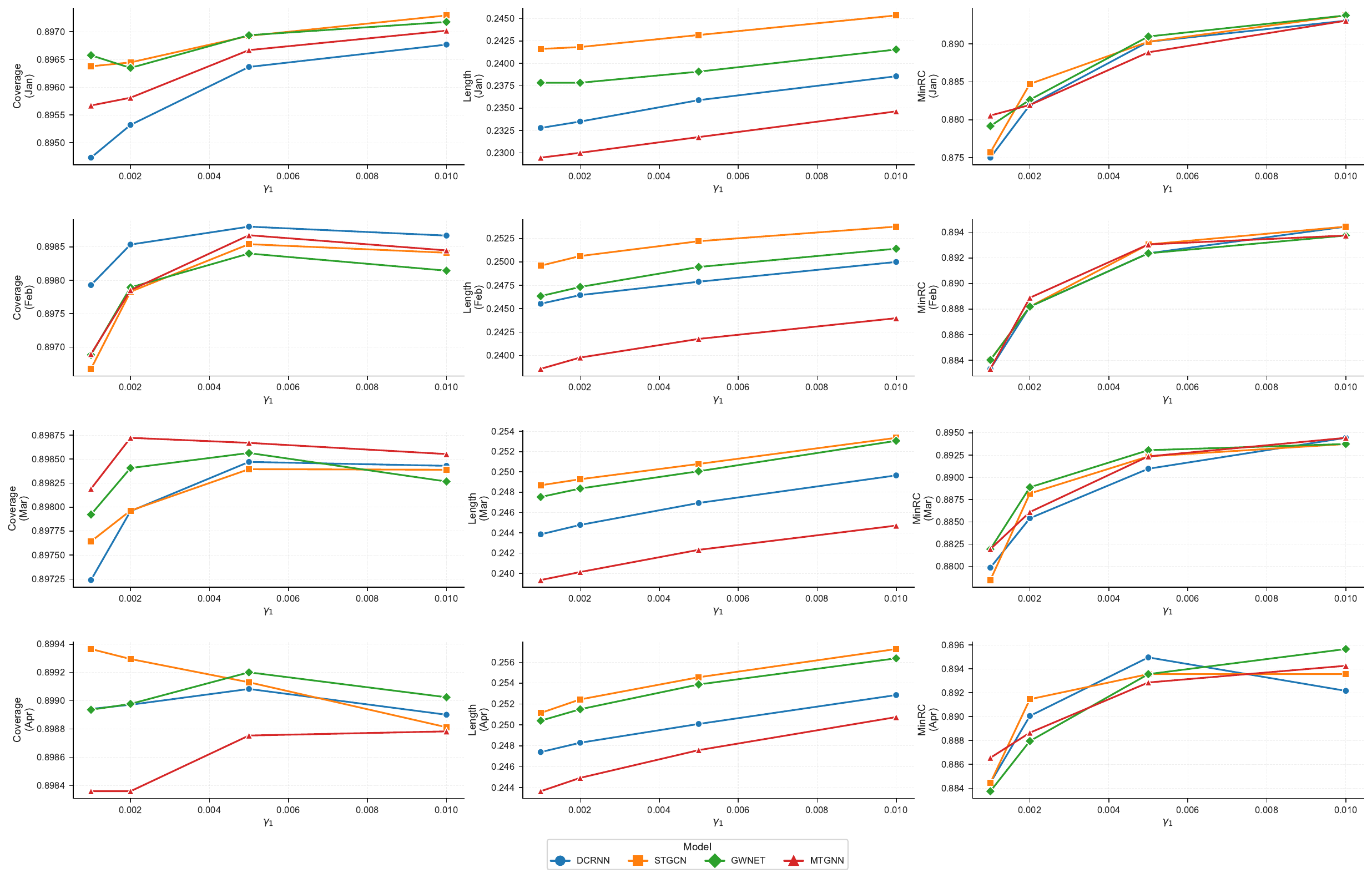}
    \caption{Results of sensitive analysis for NYCTaxi dataset}
    \label{fig:1}
\end{figure}
\begin{figure}[!h]
    \centering
    \includegraphics[width=0.8\linewidth]{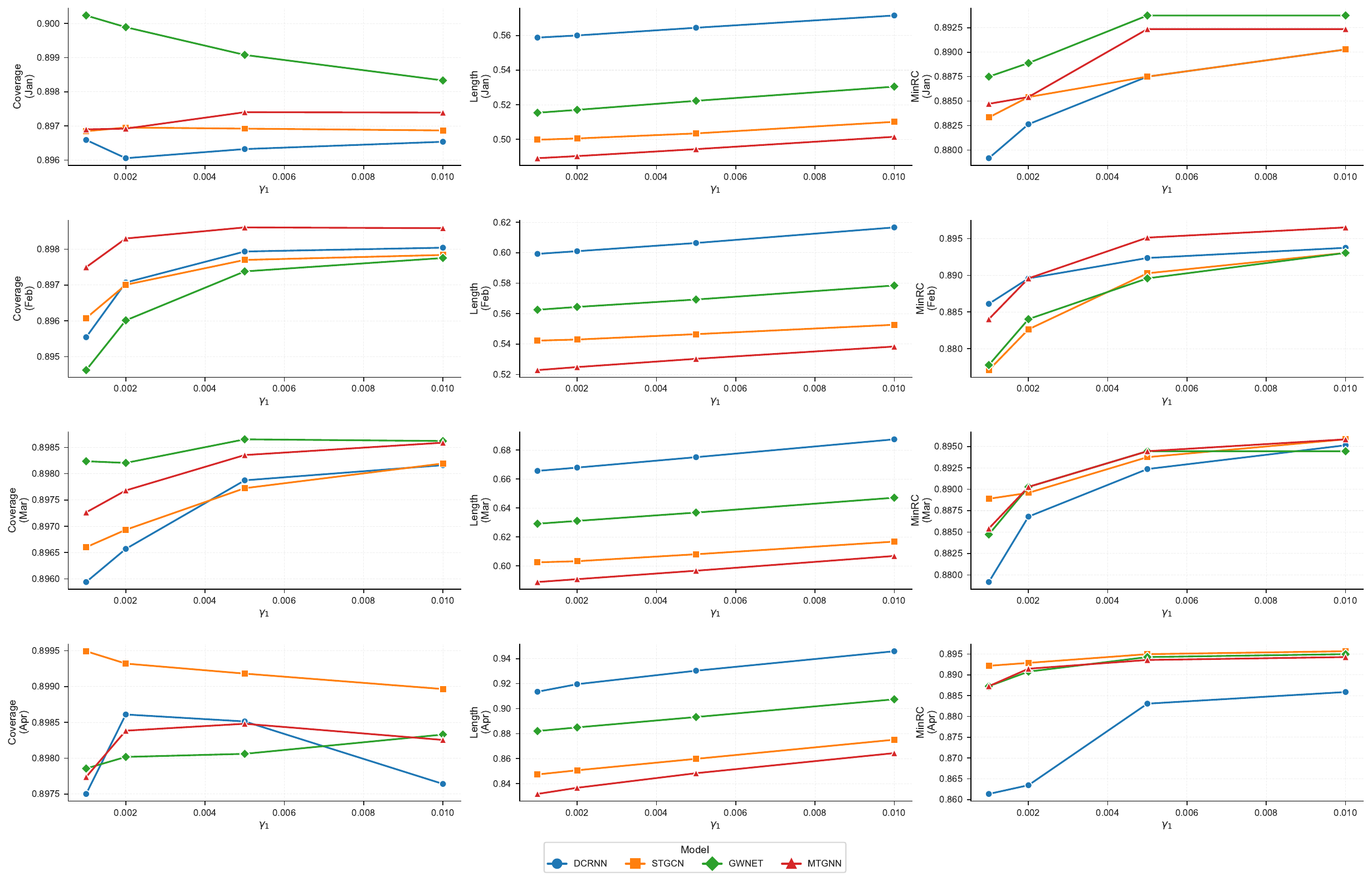}
    \caption{Results of sensitive analysis for CHIbike dataset}
    \label{fig:2}
\end{figure}
\begin{figure}[!h]
    \centering
    \includegraphics[width=0.8\linewidth]{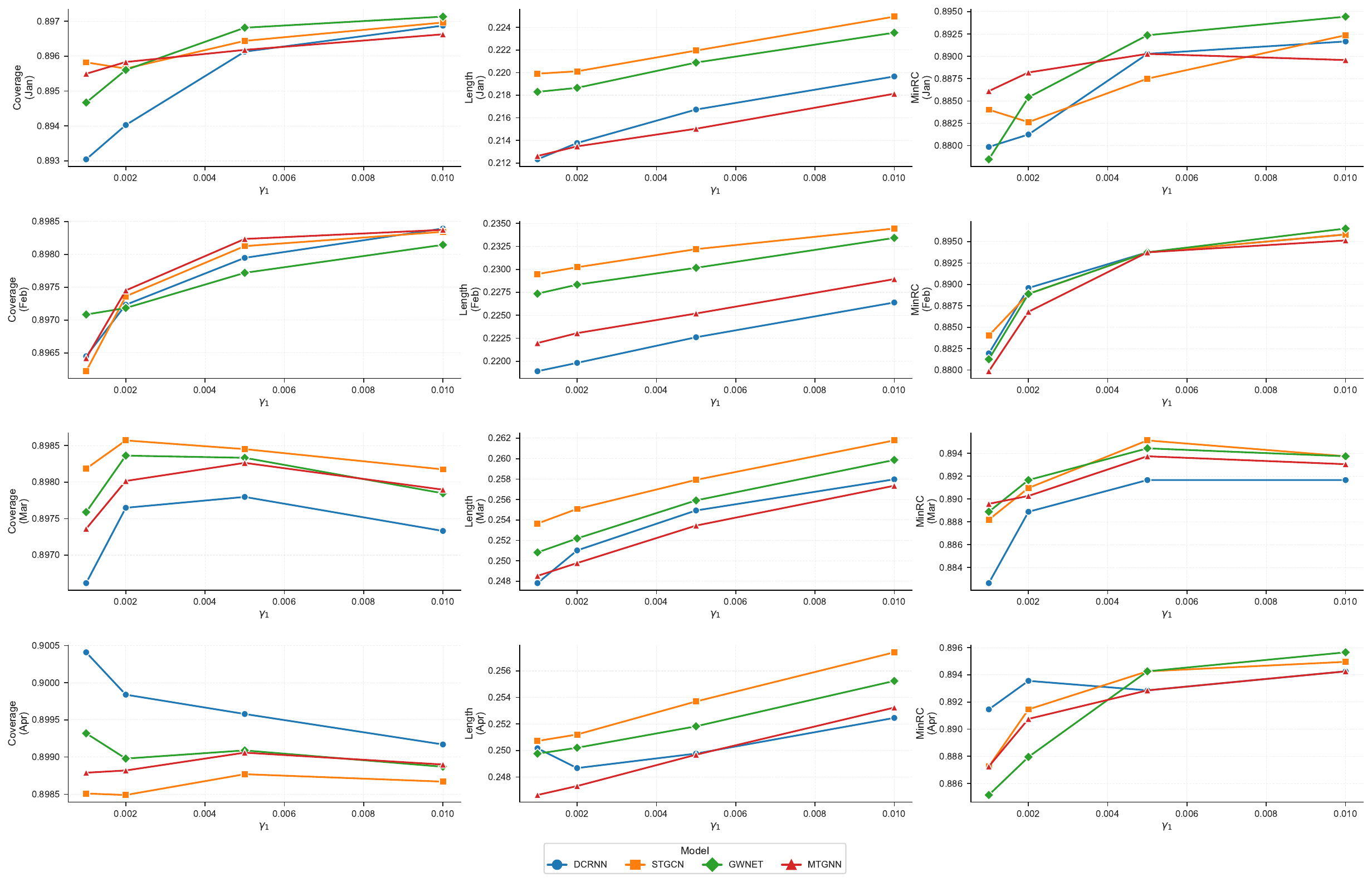}
    \caption{Results of sensitive analysis for CHIbike dataset}
    \label{fig:3}
\end{figure}

The results of sensitive analysis in NYCtaxi, CHIbike and CHItaxi datasets are similar to the results in NYCBike datasets. 
\subsection{Full result of adaptive learning rate}
\label{section:dr}
We plot the figure for STGCN, DCRNN and MTGNN in the following Figure \ref{fig:region_STGCN}, \ref{fig:region_DCRNN} and \ref{fig:region_MTGNN}.
\begin{figure}[htbp]
    \centering
    
    \subfloat[ Daily regionl coverage for NYCbike dataset]{\includegraphics[width=0.8\linewidth]{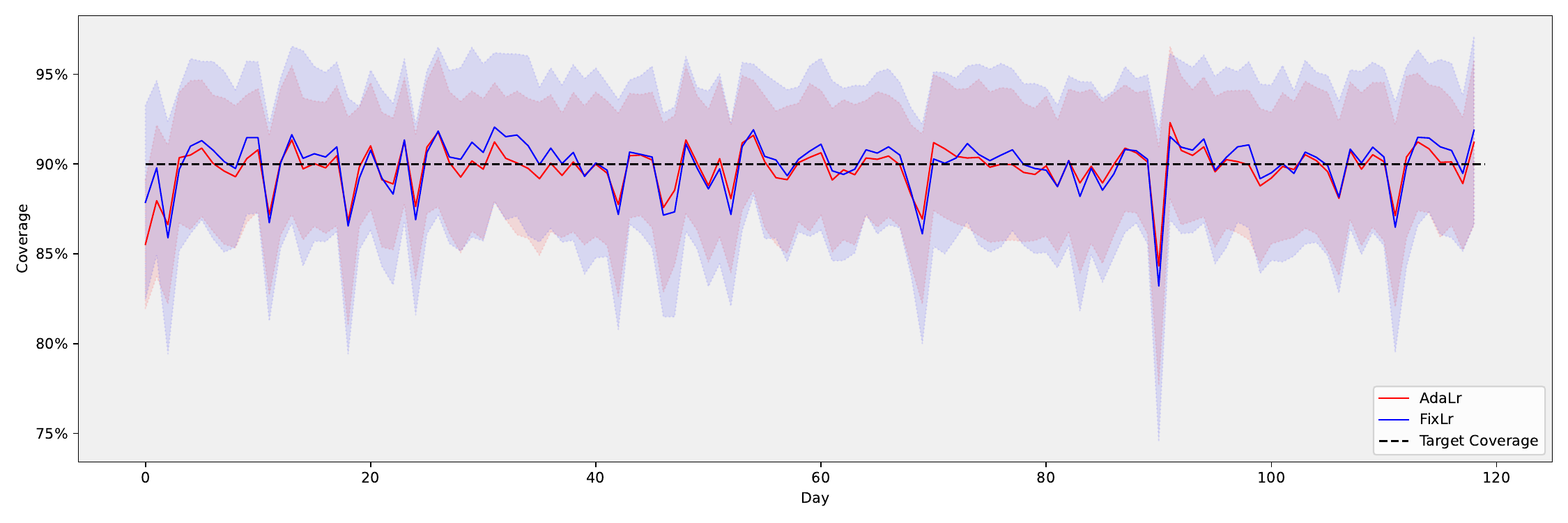}}
    \vspace{0.1cm}
    \subfloat[Daily regionl coverage for NYCatxi dataset]{\includegraphics[width=0.8\linewidth]{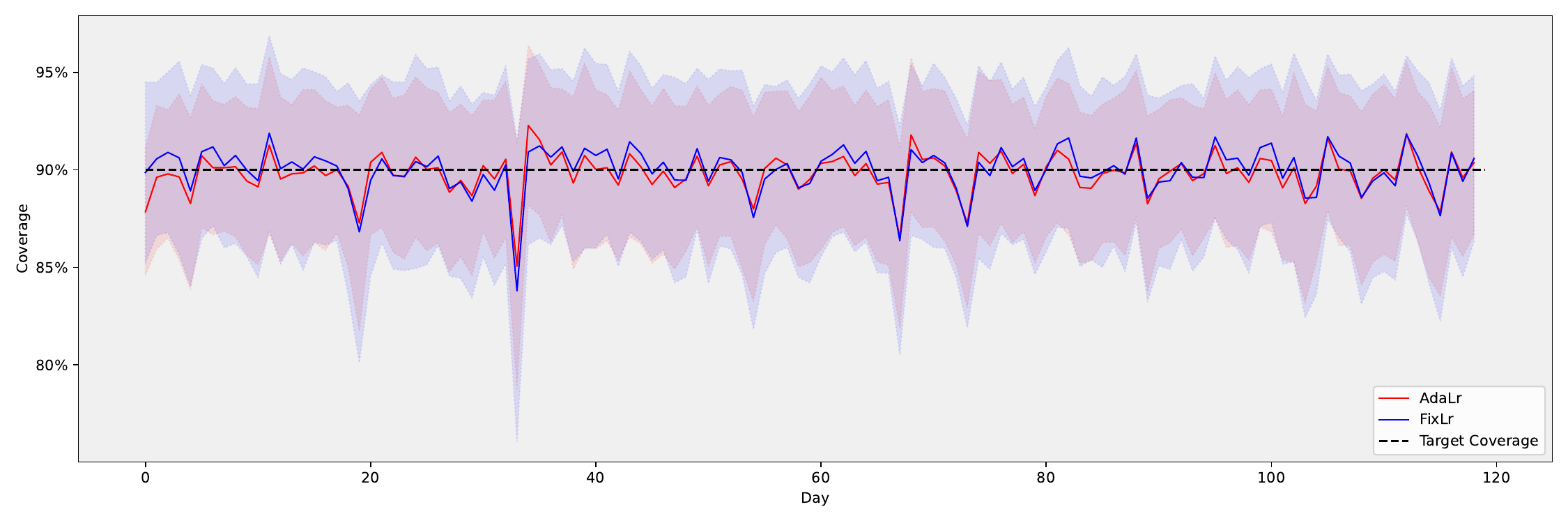}}
    \vspace{0.1cm}
    \subfloat[Daily regionl coverage for CHIbike dataset]{\includegraphics[width=0.8\linewidth]{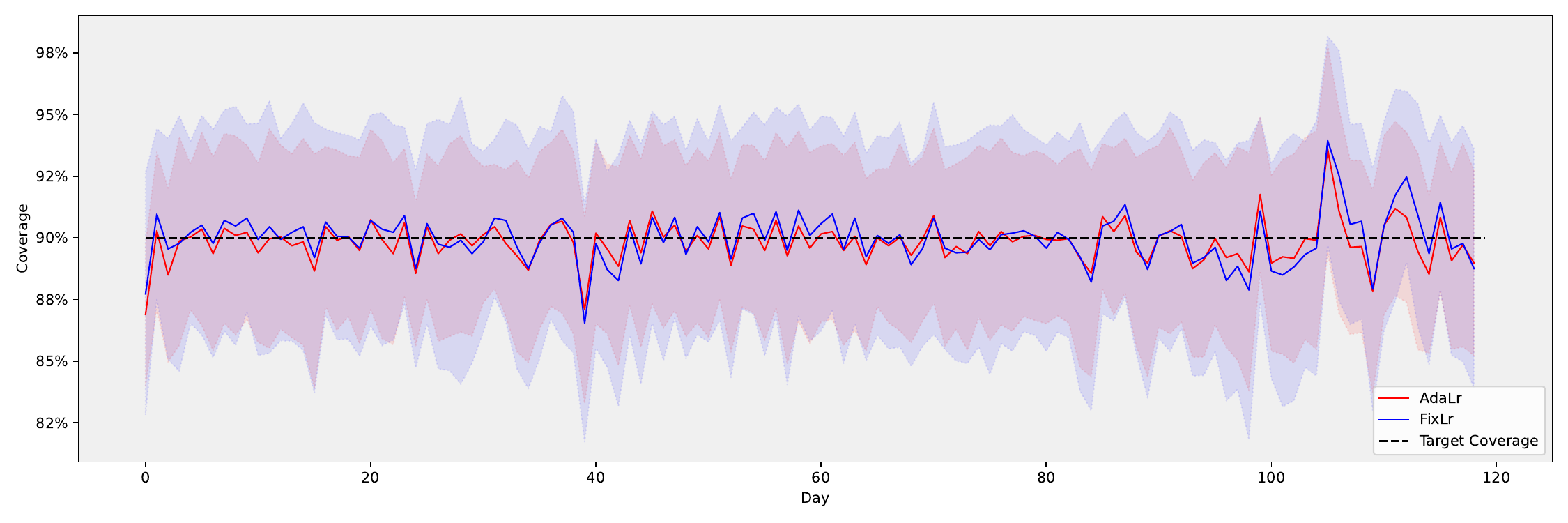}}
    \vspace{0.1cm}
    \subfloat[Daily regionl coverage for CHItaxi dataset]{\includegraphics[width=0.8\linewidth]{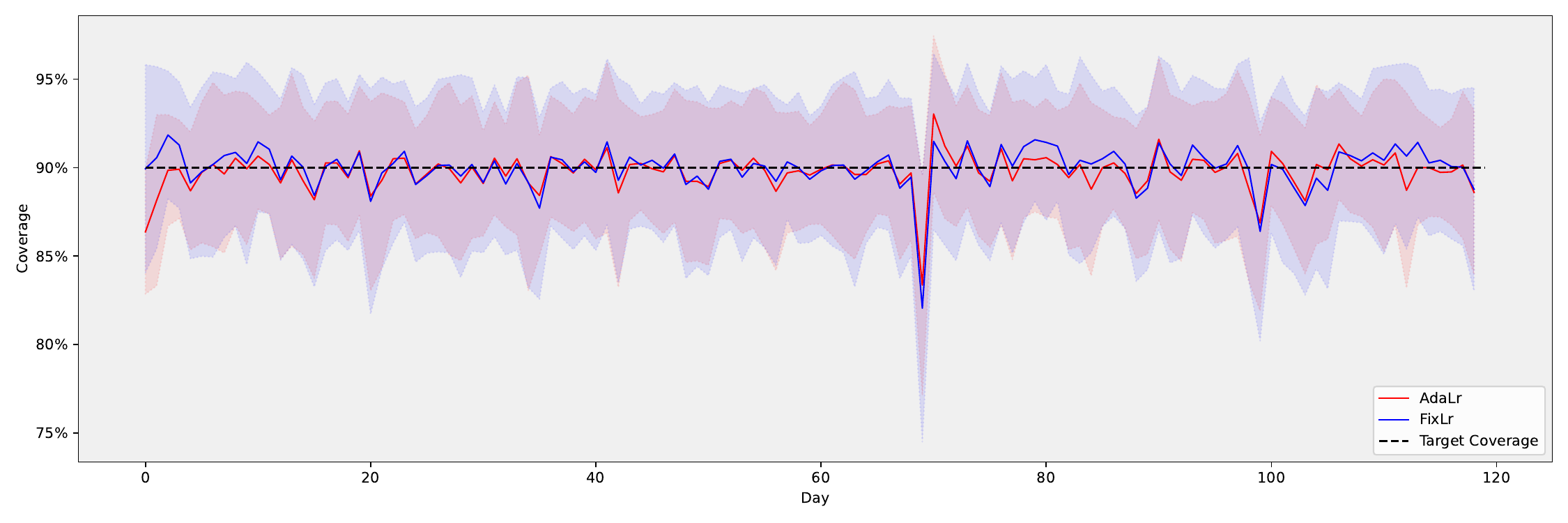}}
    \caption{Regionl coverage for STGCN model}
    \label{fig:region_STGCN}
\end{figure}
\begin{figure}[htbp]
    \centering
    
    \subfloat[ Daily regionl coverage for NYCbike dataset]{\includegraphics[width=0.8\linewidth]{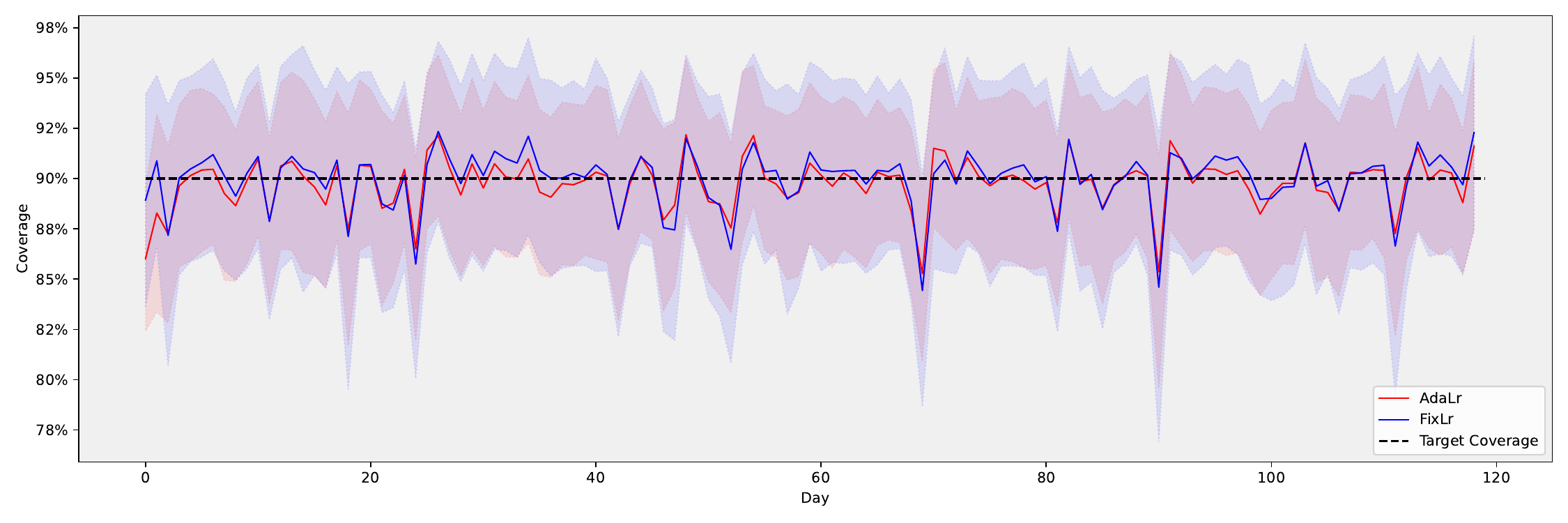}}
    \vspace{0.1cm}
    \subfloat[Daily regionl coverage for NYCatxi dataset]{\includegraphics[width=0.8\linewidth]{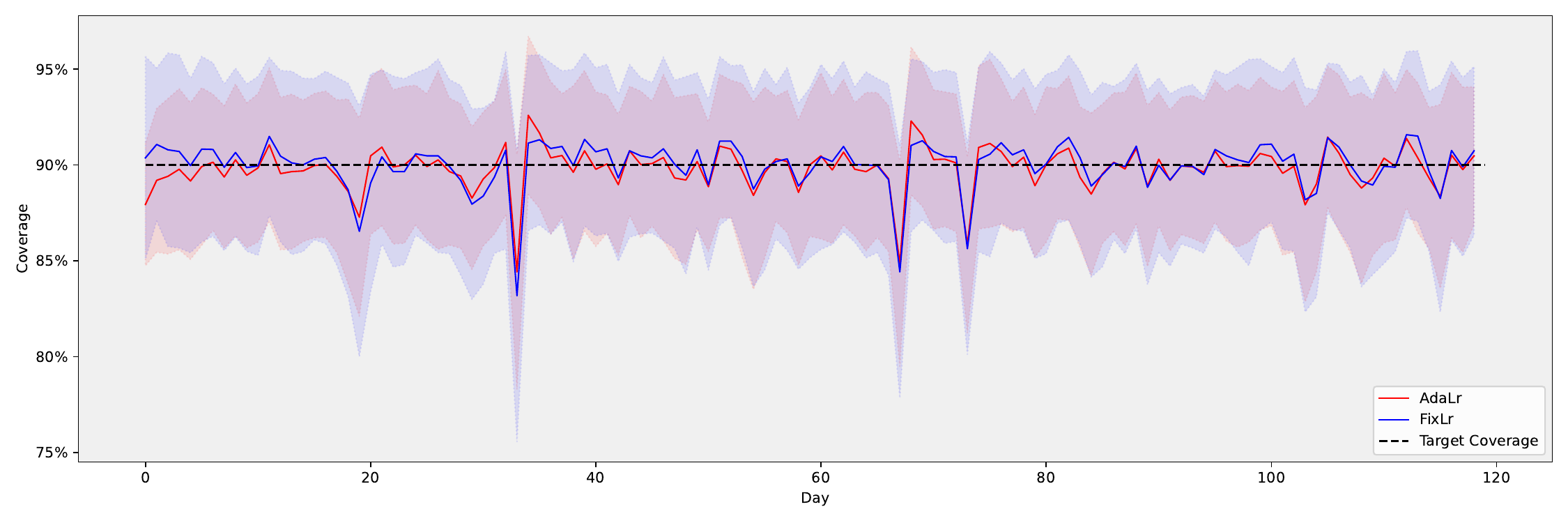}}
    \vspace{0.1cm}
    \subfloat[Daily regionl coverage for CHIbike dataset]{\includegraphics[width=0.8\linewidth]{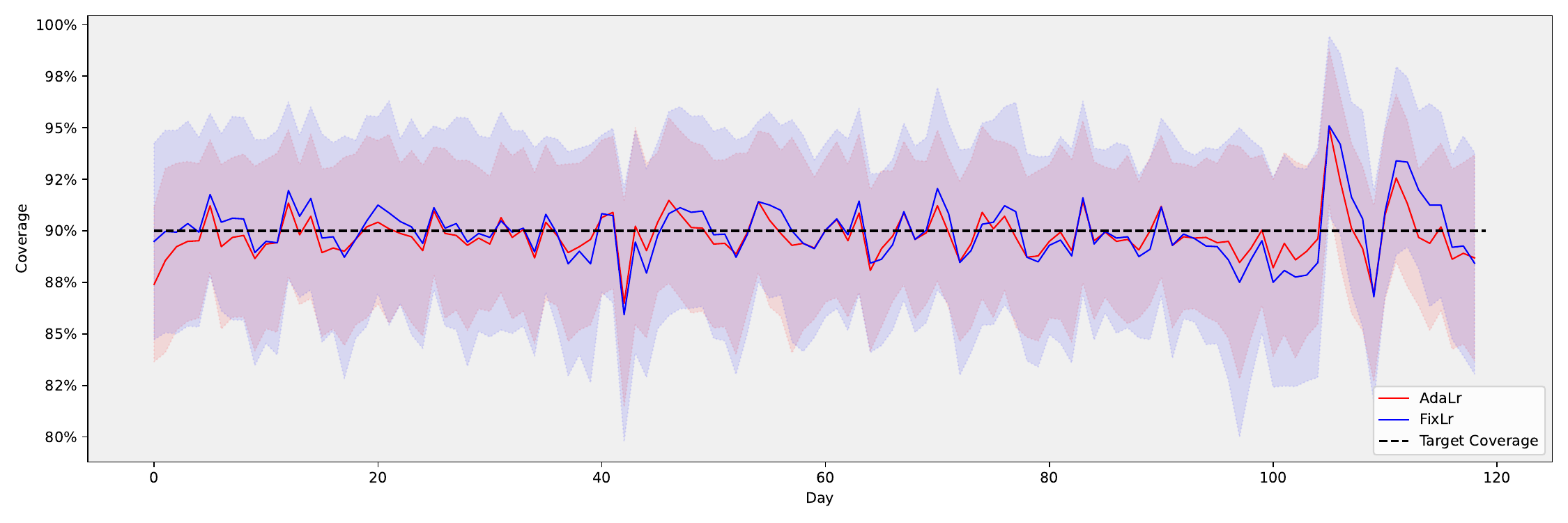}}
    \vspace{0.1cm}
    \subfloat[Daily regionl coverage for CHItaxi dataset]{\includegraphics[width=0.8\linewidth]{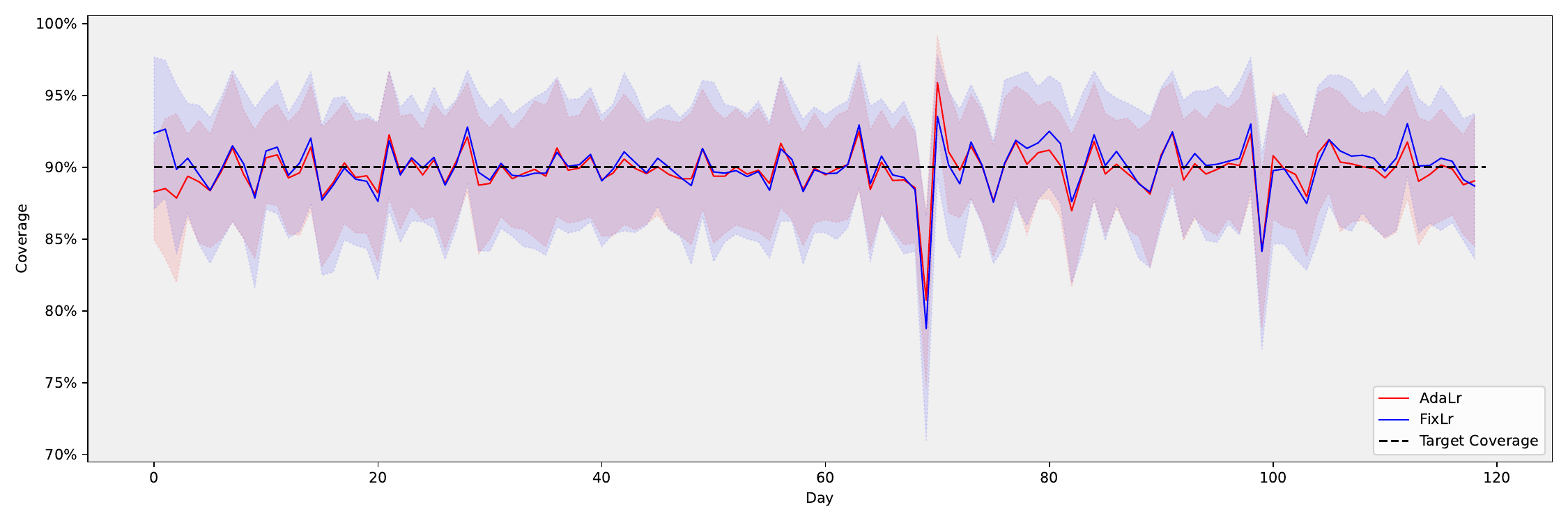}}
    \caption{Regionl coverage for DCRNN model}
    \label{fig:region_DCRNN}
\end{figure}
\begin{figure}[htbp]
    \centering
    
    \subfloat[ Daily regionl coverage for NYCbike dataset]{\includegraphics[width=0.8\linewidth]{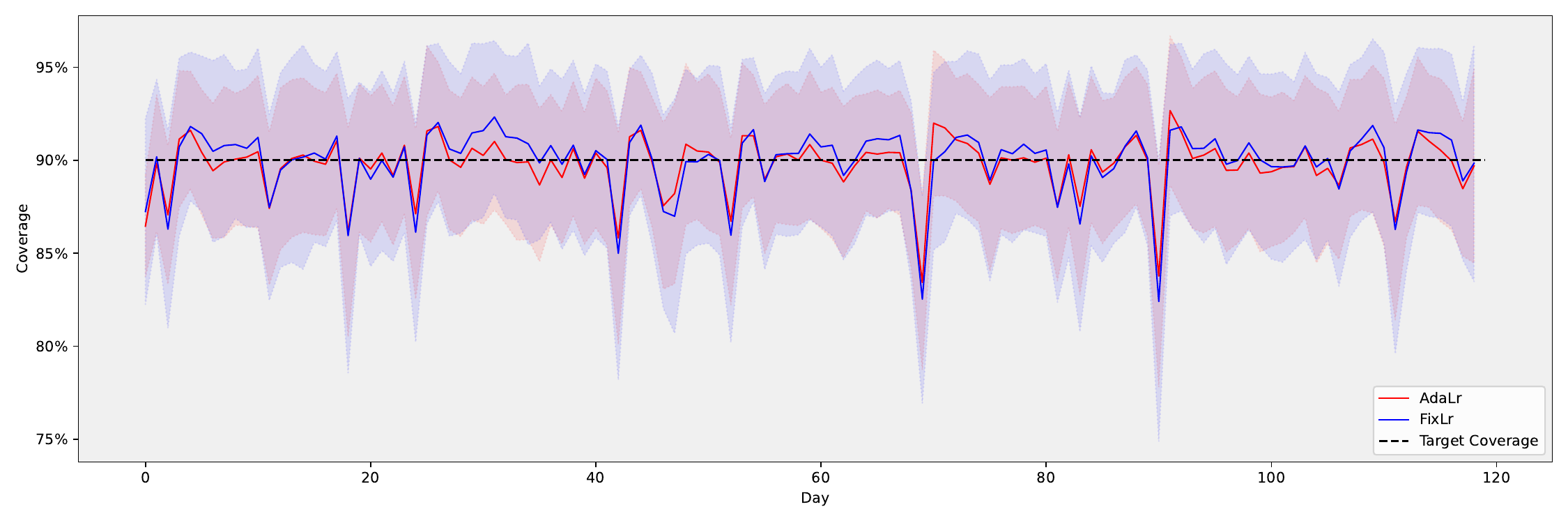}}
    \vspace{0.1cm}
    \subfloat[Daily regionl coverage for NYCatxi dataset]{\includegraphics[width=0.8\linewidth]{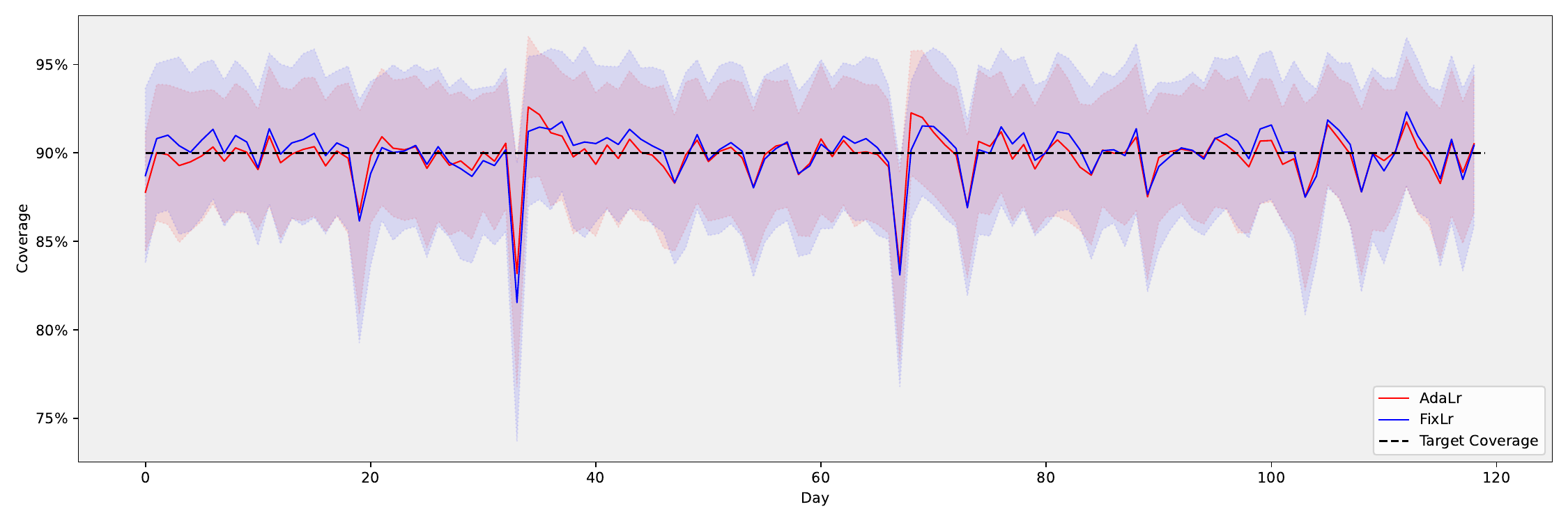}}
    \vspace{0.1cm}
    \subfloat[Daily regionl coverage for CHIbike dataset]{\includegraphics[width=0.8\linewidth]{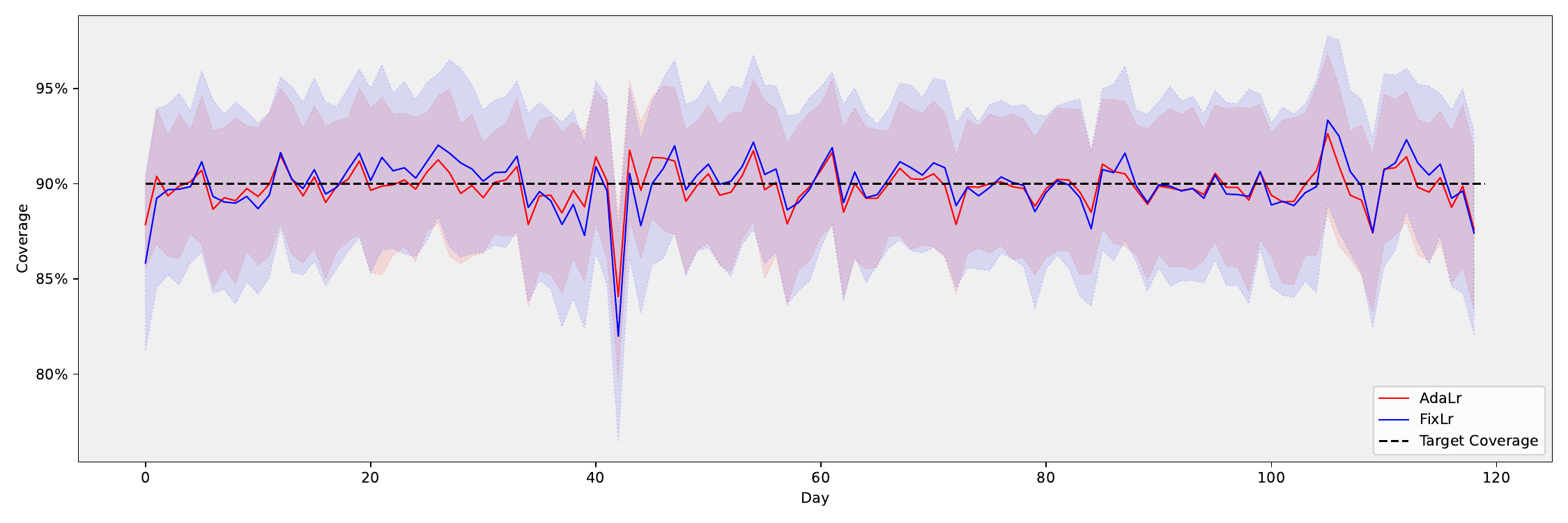}}
    \vspace{0.1cm}
    \subfloat[Daily regionl coverage for CHItaxi dataset]{\includegraphics[width=0.8\linewidth]{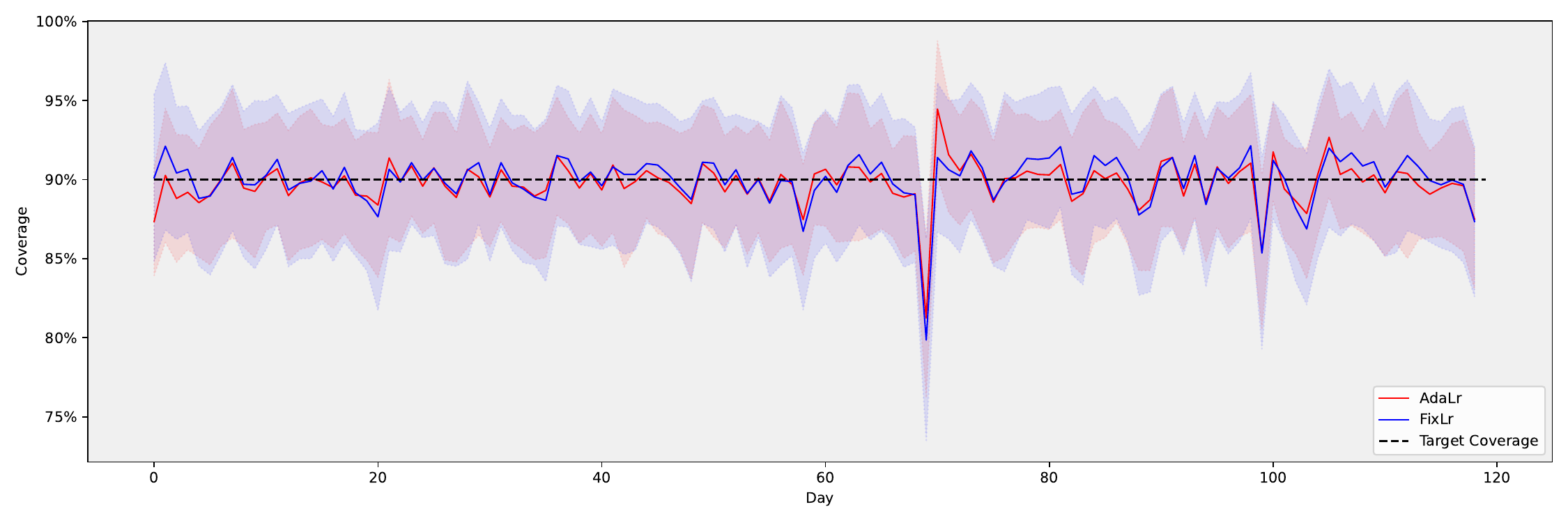}}
    \caption{Regionl coverage for MTGNN model}
    \label{fig:region_MTGNN}
\end{figure}
\clearpage
It could be found that the coverages of regions when using adaptive learning rate are more concentrated on 90\% than using fixed learning rate. And this observation is consistent with the conclusion in Section \ref{section:main_dr}.
\newpage
\end{document}